\documentclass[twoside]{article}

\usepackage[accepted]{aistats2022_arxiv}
\pdfoutput=1

\usepackage[round]{natbib}

\setcitestyle{authoryear,round,citesep={;},aysep={,},yysep={;}}
\renewcommand{\cite}{\citep}

\usepackage{amsthm,amsmath,bbm,amsfonts,amssymb}
\usepackage{dsfont} 
\usepackage{mathtools}
\usepackage{commath}
\mathtoolsset{showonlyrefs=false}

\usepackage{xspace}
\usepackage[T1]{fontenc}
\usepackage[utf8]{inputenc}
\usepackage[usenames,dvipsnames]{xcolor} 

\usepackage{hyperref,url}
\usepackage{cleveref} 

\usepackage{wrapfig}
\usepackage[export]{adjustbox}
\usepackage{algorithm,algorithmic}
\usepackage{algorithmic}

\usepackage{graphicx,tabularx}
\usepackage{multirow,hhline}
\usepackage{booktabs}

\usepackage{capt-of}

\allowdisplaybreaks 

\usepackage{pbox} 

\usepackage[export]{adjustbox}

\usepackage{amssymb}
\usepackage{pifont}
\newcommand{\cm}{\ding{51}}%

\newcommand{\xm}{\ding{55}}%

\usepackage{framed}
\usepackage[most]{tcolorbox}
\definecolor{kjgray}{rgb}{.7,.7,.7}

\usepackage{mdframed}
\usepackage{lipsum}
\definecolor{kjgray}{rgb}{.7,.7,.7}
\makeatletter



\makeatletter
\renewcommand{\paragraph}{%
  \@startsection{paragraph}{4}%
  {\z@}{0.50ex \@plus 1ex \@minus .2ex}{-1em}%
  {\normalfont\normalsize\bfseries}%
}
\makeatother

\usepackage[customcolors,shade]{hf-tikz} 

\def\ddefloop#1{\ifx\ddefloop#1\else\ddef{#1}\expandafter\ddefloop\fi}
\def\ddef#1{\expandafter\def\csname c#1\endcsname{\ensuremath{\mathcal{#1}}}}
\ddefloop ABCDEFGHIJKLMNOPQRSTUVWXYZ\ddefloop
\def\ddef#1{\expandafter\def\csname b#1\endcsname{\ensuremath{{\boldsymbol{#1}}}}}
\ddefloop ABCDEFGHIJKLMNOPQRSUVWXYZabcdefhijklmnopqrtsuvwxyz\ddefloop
\def\ddef#1{\expandafter\def\csname h#1\endcsname{\ensuremath{\hat{#1}}}}
\ddefloop ABCDEFGHIJKLMNOPQRSTUVWXYZabcdefghijklmnopqrsuvwxyz\ddefloop 
\def\ddef#1{\expandafter\def\csname hc#1\endcsname{\ensuremath{\widehat{\mathcal{#1}}}}}
\ddefloop ABCDEFGHIJKLMNOPQRSTUVWXYZ\ddefloop
\def\ddef#1{\expandafter\def\csname t#1\endcsname{\ensuremath{\tilde{#1}}}}
\ddefloop ABCDEFGHIJKLMNOPQRSTUVWXYZabcdefghijklmnopqrtsuvwxyz\ddefloop
\def\ddef#1{\expandafter\def\csname wt#1\endcsname{\ensuremath{\widetilde{#1}}}}
\ddefloop ABCDEFGHIJKLMNOPQRSTUVWXYZabcdefghijklmnopqrtsuvwxyz\ddefloop
\def\ddef#1{\expandafter\def\csname r#1\endcsname{\ensuremath{\mathring{#1}}}}
\ddefloop ABCDEFGHIJKLMNOPQRSTUVWXYZabcdefghijklmnopqrtsuvwxyz\ddefloop
\def\ddef#1{\expandafter\def\csname bar#1\endcsname{\ensuremath{\bar{#1}}}}
\ddefloop ABCDEFGHIJKLMNOPQRSTUVWXYZabcdefghijklmnopqrtsuvwxyz\ddefloop
\def\ddef#1{\expandafter\def\csname wbar#1\endcsname{\ensuremath{\overline{#1}}}}
\ddefloop ABCDEFGHIJKLMNOPQRSTUVWXYZabcdefghijklmnopqrtsuvwxyz\ddefloop
\def\ddef#1{\expandafter\def\csname tc#1\endcsname{\ensuremath{\widetilde{\mathcal{#1}}}}}
\ddefloop ABCDEFGHIJKLMNOPQRSTUVWXYZ\ddefloop

\DeclareMathOperator{\EE}{\mathbb{E}}

\DeclareMathOperator{\PP}{\mathbb{P}}

\DeclareMathOperator{\one}{\mathds{1}\hspace{-.18em}}
\DeclareMathOperator{\Reg}{{\mathsf{Reg}}}

\def\RR{{\mathbb{R}}}

\newcommand{\sr}[2]{ {\stackrel{#1}{#2}} }

\newcommand{\fr}[2]{ { \frac{#1}{#2} }}

\def\lt{\left}
\def\rt{\right}

\def\larrow{\ensuremath{\leftarrow}\xspace} 
 
\def\T{\ensuremath{\top}}  

\def\om{{\ensuremath{\omega}\xspace} }
\def\dt{{\ensuremath{\delta}\xspace} }
\def\sm{{\ensuremath{\setminus}\xspace} }

\def\barV{\ensuremath{\overline{V}}\xspace}

\def\lfl{\lfloor} 
\def\rfl{\rfloor}

\def\gam{{\ensuremath{\gamma}\xspace} }

\makeatletter
\newcommand{\vast}{\bBigg@{3}}
\newcommand{\Vast}{\bBigg@{4}}
\makeatother

\def\cX{\ensuremath{\mathcal{X}}\xspace} 
\def\la{{\langle}}
\def\ra{{\rangle}}

\def\lam{\ensuremath{\lambda}}
\def\barV{\ensuremath{\overline{{V}}\hspace{0.0em}}\xspace}

\def\til{\widetilde}

\newcommand{\wbar}[1]{ {\ensuremath{\overline{#1}}} }
\def\barbeta{\wbar{\beta}}

\def\barC{\ensuremath{\wbar{C}}}

\def\hth{{\hat{ \theta}}}

\def\lam{{\ensuremath{\lambda}\xspace} }

\def\cO{{\ensuremath{\mathcal{O}}}}

\def\tr{{\ensuremath{\normalfont{\text{tr}}}}}

\def\th{{\ensuremath{\theta}}}

\def\cd{\cdot}

\def\barH{\ensuremath{\wbar{H}}}

\newcounter{textcnt}

\newcommand\addtext[1]{%
  \stepcounter{textcnt}%
  \csgdef{text\thetextcnt}{#1}}

\newcounter{colnum}

\newcounter{Jidx}
\newcommand\dsymhelper[2]{
  \addtext{\hyperlink{#1}{#2}}%
  \blue{\hypertarget{#1}{#2}}%
}
\newcommand\dsym[1]{
  \stepcounter{Jidx}
  \xdef\tmpname{Jsym.\theJidx}
  \expandafter\dsymhelper\expandafter{\tmpname}{#1}
} 

\usepackage[utf8]{inputenc} 
\usepackage[T1]{fontenc}    
\usepackage{hyperref}       
\usepackage{url}            
\usepackage{booktabs}       
\usepackage{amsfonts}       
\usepackage{nicefrac}       
\usepackage{microtype}      
\usepackage[dvipsnames]{xcolor}         

\newcommand\myshade{85}
\colorlet{mylinkcolor}{blue}
\ifdefined\nohyperref\else\ifdefined\hypersetup
\hypersetup{ %
  pdftitle={},
  pdfauthor={},
  pdfsubject={},
  pdfkeywords={},
  pdfborder=0 0 0,
  pdfpagemode=UseNone,
  colorlinks=true,
  linkcolor=mylinkcolor!\myshade!black,
  citecolor=mylinkcolor!\myshade!black,
  filecolor=mylinkcolor!\myshade!black,
  urlcolor=mylinkcolor!\myshade!black,
  pdfview=FitH%
}

\usepackage{anyfontsize}
\usepackage{MnSymbol} 

\usepackage[shortlabels]{enumitem}

\newtheoremstyle{plain}
{3pt}   
{-3pt}   
{\itshape}  
{0pt}       
{\bfseries} 
{.}         
{5pt plus 1pt minus 1pt} 
{}          

\newtheoremstyle{plain2}
{3pt}   
{-3pt}   
{}  
{0pt}       
{\bfseries} 
{.}         
{5pt plus 1pt minus 1pt} 
{}          

\theoremstyle{plain}
\newtheorem{theorem}{Theorem}
\newtheorem{lemma}[theorem]{Lemma}
\newtheorem{corollary}[theorem]{Corollary}

\theoremstyle{plain2}
\newtheorem{remark}{Remark}

\definecolor{kjcolor}{RGB}{46,139,87}

\newcommand{\blue}[1]{{\color[rgb]{.3,.5,1}#1}}
\renewcommand{\blue}[1]{#1}

\newcommand{\removeforfinal}[1]{#1}
\renewcommand{\removeforfinal}[1]{}

\renewcommand{\ln}{\log}

\usepackage[toc,page,header]{appendix}
\usepackage{minitoc}
\usepackage{silence}
\renewcommand \thepart{}
\renewcommand \partname{}

\usepackage{bibunits}

\begin{document}
  
\doparttoc 
\faketableofcontents 


%

%

\twocolumn[

\aistatstitle{Norm-Agnostic Linear Bandits}

\aistatsauthor{  Spencer (Brady) Gales \And  Sunder Sethuraman  \And Kwang-Sung Jun}

\aistatsaddress{ University of Arizona \And  University of Arizona \And University of Arizona } ]

\begin{abstract}
  Linear bandits have a wide variety of applications including recommendation systems yet they make one strong assumption: the algorithms must know an  upper bound $S$ on the norm of the unknown parameter $\theta^*$ that governs the reward generation.
  Such an assumption forces the practitioner to guess $S$ involved in the confidence bound, leaving no choice but to wish that $\|\theta^*\|\le S$ is true to guarantee that the regret will be low.
  In this paper, we propose novel algorithms that do not require such knowledge for the first time.
  Specifically, we propose two algorithms and analyze their regret bounds: one for the changing arm set setting and the other for the fixed arm set setting.
  Our regret bound for the former shows that the price of not knowing $S$ does not affect the leading term in the regret bound and inflates only the lower order term.
  For the latter, we do not pay any price in the regret for now knowing $S$.
  Our numerical experiments show standard algorithms assuming knowledge of $S$ can fail catastrophically when $\|\theta^*\|\le S$ is not true whereas our algorithms enjoy low regret.
\end{abstract}
\section{INTRODUCTION}

Linear bandits~\cite{abe99associative} have gained popularity since their success in online news recommendation systems~\cite{li10acontextual} that learn from user feedback interactively.
Specifically, at each time step $t$, the learner chooses an arm $x_t\in\RR^d$ from an available pool of arms $\cX_t \subset\{x\in\RR^d: \|x\|\le 1\}$ given from the environment, and then receives a reward $y_t\in \RR$ that is assumed to follow a linear model $y_t  = x_t^\T\th^* + \eta_t$, where $\th^{\ast}$ is unknown and $\eta$ is a zero-mean subGaussian noise. 
This problem is an inductive extension of the classic multi-armed bandit problem~\cite{thompson33onthelikelihood} where the learner faces the dilemma between exploration and exploitation.
The goal of the algorithm is to minimize the (pseudo-)regret:
\begin{align}\label{eq:regret}
    \Reg_T = \sum_{t=1}^T \max_{x\in\cX_t} \la x, \th^*\ra - \la x_t, \th^*\ra~.
\end{align}
We refer to~\citet{lattimore20bandit} for a comprehensive review of bandit algorithms.

Despite the popularity of linear bandits, there is one unrealistic yet commonly-used assumption: the norm of the unknown vector $\th^*$ is bounded by 1 or by an a priori known constant $S$.
Such information is never known in practice.
This places the practitioner in an undesirable situation where they must make a guess on the upper bound $S$ of $\|\th^*\|$ and take a leap of faith that it will be correct. Otherwise, the algorithm does not enjoy any theoretical guarantee.
This is not just a theoretical issue as we show the failure of OFUL numerically in our experiments; see Figure~\ref{fig:expr}(c).
Such weakness of existing linear bandits raises an interesting question: Can we develop linear bandit algorithms that enjoy low regret without knowledge of $S$?

In this paper, we answer this question in the affirmative by proposing two algorithms: NAOFUL (Norm-Agnostic Optimism in the Face of Uncertainty for Linear bandits) for the changing arm set setting, and OLSOFUL (Ordinary Least Squares OFUL) for the fixed arm set setting (i.e., $\cX_t = \cX$ for some $\cX$, $\forall t$).
We show high probability regret bounds up to the time horizon $T$ in Table~\ref{tab:summary} where NAOFUL also has a free parameter $\alpha$ allowing to trade-off between the leading and the lower-order terms.
Taking an absolute constant for $\alpha$ (e.g., 2), both of our algorithms achieve the same order of regret bound, as far as the leading $\sqrt{T}$ term is concerned.
For the lower order term (i.e., scaling poly-logarithmically with $T$), NAOFUL has a slightly larger dependence on $\|\th^*\|$ in the lower-order term while OLSOFUL achieves the same order of regret bound but can only be used in the fixed arm set setting.
This indicates that the price of not knowing anything about $\|\th^*\|$ does not affect the order of the leading term in these regret bounds.

Our algorithms are both simple and computationally efficient.
Particularly, NAOFUL can be viewed as an \textit{embarrassingly simple} remedy for making OFUL adapt to the unknown norm.
Both NAOFUL and OLSOFUL can be implemented in time complexity $\cO(d^2 |\cX_t| + o(1))$ per iteration when the arm set is finite, which is just as efficient as OFUL.
Furthermore, OLSOFUL does not use any regularizer, which means that we do not have to regularize the intercept term and OLSOFUL is able to adapt to an arbitrary additive shift in the reward.
In contrast, such a shift may result in invalidating OFUL's regret bound as it has to regularize the bias term as well. 
We present NAOFUL and OLSOFUL along with their regret bounds in Section~\ref{sec:changing} and~\ref{sec:fixed} respectively.

We perform numerical evaluations to show the efficacy of our methods compared to existing algorithms in Section~\ref{sec:expr} and conclude our paper with future directions in Section~\ref{sec:conclusion}.

\begin{table*}[ht]
{
    \centering
    \begin{tabular}{cccc}\hline
        Algorithms & Norm-agnostic & Regret bound & Arm set  \\\hline
        OFUL       & \xm & $dR\sqrt{T} + S_* d$ & Changing 
    \\  NAOFUL($\alpha$)     & \cm  &  $\alpha dR\sqrt{ T} + \alpha
    dS_*+ S_*^{2/\alpha+1}$  & Changing   
    \\  OLSOFUL    & \cm & $dR\sqrt{T} + S_* d$ & Fixed\\\hline
    \end{tabular}
    \caption{Comparison of OFUL and our algorithms. The regret bounds are expressed order-wise only, ignoring polylogarithmic factors. The regret bound of OFUL here is based on an improved analysis that was not reported in the literature explicitly, which we report in our appendix. For NAOFUL$(\alpha)$, we choose a polynomially decaying regularizer (see Corollary~\ref{cor:changing}(ii)).}
    \label{tab:summary}
    }
\end{table*}

\section{RELATED WORK}

For many online learning problems with bandit feedback, it has been common to assume the knowledge of $S$ such that $\|\th^*\|\le S$~\cite{abe99associative,auer02using,ay11improved}.
Some papers further assume that the mean rewards $\{x_t^\T\th^*\}$ are in a known interval such as $[-1,1]$~\cite{ay11improved}.
The assumption on the norm is highly unrealistic, yet researchers mostly left them untouched except for a few studies such as \citet[Section 4]{orabona11better}, \citet[Theorem 4]{gentile14onmultilabel}, and \citet[Remark 9]{kocak20spectral}.
These works, however, result in an exponential dependency on $\|\th^*\|$ in the lower-order term of their performance bound, which is undesirable.
For the fixed arm set setting, \citet{dani08stochastic} does not assume knowledge of $S$ but does assume that the mean rewards are known to be in $[-1,1]$.

One may attempt to employ model selection to adapt to the unknown norm of $\theta^*$.
\cite{ghosh2021problem} consider the $K$ arm mixture bandit algorithm with bias. They do not consider a bound on $\|\th^*\|$ nor expected rewards. Using a model selection approach, to estimate an upper bound on $\|\theta^*\|$, they incur $\tilde{\cO}((1+\|\th^*\|)(\sqrt{K}+\sqrt{d})\sqrt{T}).$ \cite{wang2021neural} estimate the magnitude of a data dependent complexity term (similar to $\|\theta^*\|$) required in there deep active learning algorithm. 
However, these works consider quite a different problem setup from ours.
Furthermore, it is unclear if the model selection approach will enjoy the same bound as ours since many recent results on model selection in bandits indicate that the cost of adaptation is nontrivial and results in a nontrivial factor to the leading term of the regret bound~\cite{zhu20regret,zhu21pareto}.

Algorithms with Bayesian regret guarantees do not require the knowledge of the unknown norm of $\theta^*$~\cite{russo14learning}.
However, these algorithms in fact require a stronger assumption than knowing the norm: they typically assume the knowledge of the prior over the unknown $\theta^*$ to enjoy a valid regret bound.
We remark that our norm-agnostic algorithms still enjoy the claimed regret bound even in the Bayesian environment where $\th^*$ is drawn from a prior distribution, without the knowledge of the prior.

In sum, our work is the first work in the literature to not assume knowledge of $S$ for the changing arm setting nor the range of the mean rewards without introducing an exponential dependence on $\|\th^*\|$, to our knowledge.

\section{BACKGROUNDS} \label{sec:backgrounds}
Let us now formally define the linear bandit problem with changing arm set.
At each time step $t$, the learner chooses an arm $x_t\in\cX_t\subset\RR^d$ from an available pool of arms $\cX_t\subset\{x\in\RR^d: \|x\|\le 1\}$ given from the environment and then receives a reward $y_t\in \RR$ that we assume to have the following structure $y_t  = x_t^\T\th^* + \eta_t$ where $\theta^*\in\RR^d$ is an unknown parameter and $\eta_t$ is a zero-mean $R^2$-sub-Gaussian noise conditioned on all the previous observations; i.e., $\forall v, \EE[\exp(v \eta_t) \mid x_{1:t}, \eta_{1:t-1}] \le \exp(v^2 R^2/2)$ where we denote $z_1,...,z_t$ by $z_{1:t}$.

\textbf{Assumptions on the arm sets.}
Our theoretical bounds will still be valid when $\cX_t$ depends on anything that happens up to the end of the time step $t-1$.
However, such a generic setting makes an interpretability issue on the regret (Eq.~\eqref{eq:regret}) since the comparator $\max_{x\in\cX_t} x^\T\th^*$ in the definition of regret is now dependent on the algorithm's behavior. 
To avoid any confusion, we assume that, for each $t\ge 1$, the arm set $\cX_t$ is determined by the environment before the game begins and revealed to the learner at the beginning of the time step $t$.
Note that the fixed arm set setting is a simpler setup where $\cX_t = \cX, \forall t$ and $\cX$ is known to the learner before the game starts.

\textbf{Notations.}
Define $S_* = \|\th^*\|$.
The order symbol $\tilde\cO$ ignores polylogarithmic factors.
For the matrix $A\in\RR^{d\times d}$, define $|A|$ to be the determinant of $A.$ For the set $\cX\subset \RR^d$, define $|\cX|$ to be the cardinality of $\cX$. 

\textbf{OFUL.}
One of the most popular and performant linear bandit algorithm is OFUL~\cite{ay11improved}.
We present OFUL in Algorithm~\ref{alg:oful} where
\begin{align}
    V_t(\lambda)=\sum_{s=1}^t x_sx_s^{\T}+\lambda I~.
\end{align}
OFUL maintains a regularized MLE 
\begin{align}\label{eq:mle}
    \hth_t = \arg\min_{\th} \sum_{s=1}^t \fr12(y_s - x_s^\T\th)^2 + \frac{\lam}{2} \|\th\|^2
\end{align}
and a confidence set around it.
OFUL then chooses an arm by the celebrated optimism principle (Eq.~\eqref{eq:optimism}) where we choose an arm that maximizes the upper confidence bound on its mean reward.
The original guarantee of OFUL along with proper tuning of $\lam$ is as follows:
\begin{theorem}[Theorem 3 of \cite{ay11improved}] \label{thm:oful}
Assume that the hyperparameter $S$ of OFUL (Algorithm~\ref{alg:oful}) satisfies $S_* \le S$.
If $\langle x, \th^*\rangle \in [-1,1], \forall x\in\cX_t, t\in[T]$, then with $\lambda=\frac{R}{S^2}$, OFUL satisfies, w.p. at least $1-\dt$,
\begin{equation}\label{OFUL_reg_bound}
    \Reg_T = \tilde{\cO}(dR\sqrt{T}).
\end{equation}
\end{theorem}
Throughout the paper, all the proofs are deferred to the appendix unless noted otherwise.
We remark that we hide the dependency on $\dt$ to reduce clutter but all the regret bounds in our paper have a multiplicative factor of $\sqrt{\log(1/\dt)}$ in the leading term.

However, Theorem~\ref{thm:oful} is in fact missing an additive lower order term $\Omega(d S^*)$, which we elaborate in Appendix~\ref{sec:error} along with a minor error in~\citet{ay11improved}.
The key idea is that even if there is no noise (i.e., $R=0$), there exists a problem where OFUL is forced to pulling $d-1$ suboptimal arms for the first $d-1$ iterations, incurring $S^*$ instantaneous regret at each time step.

Furthermore, the assumption in Theorem~\ref{thm:oful} that $\la x,\th^*\ra \in [-1,1]$ is restrictive and hard to verify in practice.
One can drop such an assumption and prove a regret bound using the same proof structure as for OFUL, giving a regret bound of $\tcO((1+S_*) dR\sqrt{T})$ in terms of an extra factor of $1+S_*$ which seems restrictive.
In fact, there exists a novel technique sketched in~\citet[Exercise 19.3]{lattimore20bandit} that removes such a factor.
\begin{theorem}
    Assume that the hyperparameter $S$ of OFUL (Algorithm~\ref{alg:oful}) satisfies $S_* \le S$.
    Let $\lambda=\frac{R}{S^2}$.
    Then, OFUL satisfies, w.p. at least $1-\dt$,
  \begin{align*}
      \Reg_T = \tilde{\cO}(dR\sqrt{T} + S_* d) 
  \end{align*}
  where the dependence on $S$ is logarithmic and thus omitted.
\end{theorem}
This matches the lower bound~\cite{dani08stochastic}, as far as the leading term is concerned.
The lower order term of $S_* d$ is not restrictive -- one can extend the argument described in Appendix~\ref{sec:error} to any algorithm.
That is, one can show that there exists a set of problems where any algorithm needs to make at least $\Theta(d)$ suboptimal arm pulls before identifying the best arm in one of those problems where each suboptimal arm pulls incur an instantaneous regret of $S_*$.

\begin{algorithm}[t]
  \caption{OFUL~\cite{ay11improved}}
  \label{alg:oful}
  \begin{algorithmic}[1]
    \STATE \textbf{Input:} $S>0$, $\lam>0$
    \FOR{$t=1,2,\ldots$}
    \STATE Compute $\hth_{t-1}$ by~Eq.~\eqref{eq:mle}.
    \STATE $\sqrt{\beta_{t-1}(\lam)} 
        :=         R\sqrt{\log\lt(\frac{{|V_{t-1}(\lam)|}}{{|\lambda I|}\delta^2}\rt)}+\sqrt{\lam}S$.
    \STATE Pull 
    \begin{align}\label{eq:optimism}
        x_t = \arg \max_{x\in\cX_t} \max_{\th\in C_{t-1}} \la x, \th\ra 
    \end{align} 
     where
    \begin{align*}
        C_{t-1} = \Big\{ \th\in\RR^d&: \|\hat{\theta}_{t-1}-\th\|_{V_{t-1}(\lam)}\leq \sqrt{\beta_{t-1}(\lam)} 
        \Big\}~.
    \end{align*}
    \STATE Receive reward $y_t\in\RR^d$.
    \ENDFOR
  \end{algorithmic}
\end{algorithm} 

\begin{algorithm}[t]
  \caption{NAOFUL (Norm-Agnostic OFUL)}
  \label{alg:naoful}
  \begin{algorithmic}[1]
    \STATE \textbf{Input:} a nonincreasing sequence $\{\lambda_t>0\}_{t\ge 1}$
    \STATE Let $\delta_t=\frac{1}{t^2}\frac{6}{\pi^2}>0, \forall t \ge 1$.
    \FOR{$t=1,2,\ldots$}
    \STATE Compute $\hth_{t-1}$ by~Eq.~\eqref{eq:mle}.
    \STATE $\sqrt{\barbeta_{t-1}} := R\sqrt{\log\lt(\frac{{|V_{t-1}(\lambda_t)|}}{{|\lambda_t I|}{\delta_t^2}}\rt)}+R\sqrt{d}$
    \STATE Pull $x_t = \arg \max_{x\in\cX_t} \max_{\th\in \barC_{t-1}} \la x, \th\ra $ where
    \begin{align}
        \barC_{t-1} = \cbr{ \th\in\RR^d: \|\hat\theta_{t-1}-\th\|_{V_{t-1}(\lambda_t)}\leq
        \sqrt{\barbeta_{t-1}} }~.
    \end{align}
    \STATE Receive reward $y_t\in\RR^d$.
    \ENDFOR
  \end{algorithmic}
\end{algorithm} 

\section{THE CHANGING ARM SET}
\label{sec:changing}

To lift the common assumption of knowing a bound $S$ on $\|\th^*\|$, we propose a new algorithm called NAOFUL (Norm Agnostic OFUL) described in Algorithm~\ref{alg:naoful}.
Our algorithm makes two simple adjustments from OFUL: introduction of (i) time dependent regularizer $\lam_t$ and (ii) a novel confidence width $\sqrt{\barbeta_t}$.
We emphasize that $\sqrt{\barbeta_t}$ does not involve $S$ or $S_*$, which means that the algorithm does not require any knowledge on $\th^*$.
One can show that to obtain a valid version of the confidence ellipsoid for $\th^*$ at any time step $t$, one needs to use $\sqrt{\lam_t} S_*$ in place of $R\sqrt{d}$ in the definition of $\sqrt{\barbeta_{t-1}}$.
The key idea is that there exists a critical time step $t_0$ after which $\sqrt{\lam} S_* \le R\sqrt{d}$ becomes true so that the confidence ellipsoid is true.
This leaves us to analyze the regret for a time-varying regularizer and the regret before the time step $t_0$.
NAOFUL enjoys the following regret bound.
\begin{theorem}\label{thm:naoful}
    Suppose we run \textup{NAOFUL} (Algorithm~\ref{alg:naoful}) with a nonincreasing sequence $\{\lam_t > 0\}_{t\ge1}$.
    Let $t_0:=\min\{t\ge1: \lambda_t^{1/2}S_* \le R\sqrt{d}\}$.
    Then, for any $T\ge1$, with probability at least $1-\delta$,
    \[\Reg_T \!=\! \cO\del{t_0S_* \!+\! R\sqrt{T }\cd d\log\del{\frac{T}{\lam_T}} \!+\! S_* d\log\left(\frac{1}{\lambda_T}\right)}\]
\end{theorem}
Theorem~\ref{thm:naoful} is stated for a general sequence $\{\lam_t\}$. 
In the corollary below, we provide a few choices.
\begin{corollary}\label{cor:changing}
Under the same assumption as Theorem~\ref{thm:naoful}, 
\begin{enumerate}
    \item[(i)] If $\lambda_t=\frac{R^2 d}{\log^\gam(1+t)}$ for $\gam>0$, w.p. at least $1-\dt$,
      \[\Reg_T= \tcO(Rd\sqrt{T} + S_*(d + \exp(S_*^{2/\gam})))\] 
    \item[(ii)] If $\lambda_t=\frac{R^2 d}{t^\alpha}$ for $\alpha\geq 0$, w.p. at least $1-\dt$,
      \[\Reg_T= \tcO(Rd\sqrt{(\alpha^2+1)T} + S_*(d  \alpha + S_*^{2/\alpha})) \] 
    \item[(iii)] If $\lambda_t=\frac{R^2 d}{\exp(t^q)}$ for $q\in(0,1)$, w.p. at least $1-\dt$,
      \[\Reg_T= \tcO(Rd\sqrt{T^{1+q}}+S_*(d T^q + \log^{1/q}(S_*^2)))\] 
\end{enumerate}
\end{corollary}
Various choices of $\{\lam_t\}$ exhibit different tradeoffs between the leading term (involving $\sqrt{T}$) and the lower order term (involving $S_*$) in the regret bound.
If the practitioner knows a rough range of $S_*$ and the time horizon $T$, one can make an informed choice accordingly.
Specifically, if $R\sqrt{T}$ is much larger than $S_*$ then the rate $\lam_t = \cO(1/\log^\gam(t))$ is preferred whereas if $S_* \gg R T^{1-(q/2)}$ then the rate $\lam_t = \exp(-t^q)$ is preferred.
If there is absolutely no information on $S_*$, then the polynomial decay $R^2 d / t^\alpha$ seems most reasonable because it does not introduce an exponential dependence on $S_*$ and does not change the order w.r.t. $T$.

\textbf{Computational complexity.}
Note that when the arm set is finite, OFUL can be implemented with $\cO(d^2|\cX_t|)$ time complexity per time step by updating $V_t^{-1}(\lam)$ using the matrix inversion lemma and $\log(|V_{t}(\lam) |)=\log(|V_{t-1}(\lam) |) + \log(1 + \|x_t\|^2_{V_{t-1}^{-1}(\lam)})$.
Such updates are nontrivial for NAOFUL since it has a time-varying $\lam_t$.
However, for the polynomial decay of $\lam_t = \fr{R^2 d}{t^\alpha}$, one can achieve a similar decay by a piecewise-linear schedule of $\lam'_t = R^2 d \exp(-\lfl \alpha \log(t) \rfl)$ without changing the regret guarantee. 
This way, we can make the same efficient updates as OFUL while $\lam'_t$ does not change.
Up to time $T$, the regularizer $\lam'_t$ changes only $\cO(\log(T))$ times, and when such a change happens we pay $\cO(d^3)$ for recomputing the inversion.
Therefore, the average per-time-step time complexity is $\cO(d^2|\cX_t| + d^3 \log(T)/T)$ where the latter term disappears quickly as $T$ gets large, making NAOFUL as efficient as OFUL.

\begin{remark}
We remark that Theorem~\ref{thm:naoful} can be further tightened by replacing the definition of $t_0$ to $t'_0 := \min\{t\ge1: \lam \|\th^*\|_{V_t(\lam_t)^{-1}} \le R\sqrt{d}\}$. 
However, such a bound involves an algorithm-dependent quantity $V_t(\lam_t)$ and does not provide further insight.
We speculate that it might be possible to bound $t'_0$ by a function of the sequence of the arm sets $\{\cX_t\}$ and improve the bound in Theorem~\ref{thm:naoful}.
Such an investigation is beyond the scope of our paper, which we leave as future work.
\end{remark}

\subsection{Proof of Theorem~\ref{thm:naoful}}

We provide a proof sketch here and defer the full proof to our appendix.
The proof of NAOFUL differs from that of~\citet{ay11improved} in two ways.
First, we separate the analysis by whether the confidence bound is true at time $t$ or not.
Second, we leverage the result of~\citet[Exercise 19.3]{lattimore20bandit}, which we call the \textit{elliptical potential count} lemma and present its improved and generalized version in the appendix.
This way, we do not introduce a factor of $1\vee S_*$ in the leading term of the regret when dropping the assumption of~\citet{ay11improved} that $\la x_t, \th^*\ra \in [-1,1], \forall t$.

Let $r_t := \max_{x\in\cX_t} \la x, \th^*\ra - \la x_t, \th^*\ra$ and $t_0 = \min\{t: \lam_t^{1/2}S_* \le R\sqrt{d}\}$.
Then, using \citet[Theorem 2]{ay11improved} with a union bound over $t\ge1$, the confidence bound at $t\ge t_0$ is true; see Theorem~\ref{Thm: Confidence_ell} in our appendix.
Then,
\begin{align*}
    \Reg_T 
    &= \sum_{t=1}^{t_0-1} r_t + \sum_{t=t_0}^T r_t   
    \\&\le t_0 2S_* + \sum_{t=t_0}^T r_t    \tag{$\|x\|\le 1, \forall x\in \cX_t, \forall t$} 
\end{align*}
It remains to bound the second term.
Define $H_T := \{t\in \{t_0,\dots,T\}: \|x_t\|_{V_{t-1}^{-1}(\lambda_T)}^2 > 1\}$ and $\barH_T := \{t_0,\ldots,T\} \sm H_T$.
Then, 
\begin{align}
  &\sum_{t=t_0}^T r_t \notag
\\&= \sum_{t=t_0}^T r_t \one\cbr[0]{\|x_t\|^2_{V^{-1}_{t-1}(\lam_T) } \leq 1 } 
     + \sum_{t=t_0}^T r_t \one\cbr[0]{\|x_t\|^2_{V^{-1}_{t-1}(\lam_T) } > 1 }  \notag
\\&\le \sum_{t=t_0}^T r_t \one\cbr[0]{\|x_t\|^2_{V^{-1}_{t-1}(\lam_T) } \leq 1 } 
     + 2S_* |H_T| \notag
\\&\le \sqrt{|\barH_T| \sum_{t\in \barH_T} r_t^2} 
     + 2S_* |H_T| \label{eq:proof-thm3-lastline}
\end{align}
where the last line is due to Cauchy-Schwarz.
The standard analysis of linear bandits~\cite{ay11improved} implies that (i) $r_t \le 2 \sqrt{\barbeta_{t-1}} \|x_t\|_{V_{t-1}^{-1}(\lam_t)} \le 2 \sqrt{\barbeta_{t-1}} \|x_t\|_{V_{t-1}^{-1}(\lam_T)}$ and (ii) $\sqrt{\barbeta_{t-1}} \le R\sqrt{d \log(1 + \fr{t-1}{d\lam_t}) + 2\ln(\fr{1}{\dt_t})  } + R\sqrt{d}
= \cO( R\sqrt{d \log(T (1+ \lam_T^{-1}))}) $ where the last inequality holds as $t\le T$ and both $\{\lam_t\}$ and $\{\dt_t\}$ are nonincreasing. 
(Recall that we ignore the dependence on $\dt$ in the order notation.)
Then,
\begin{align*}
  \sum_{t\in \barH_T} r_t^2
    &\le \sum_{t\in\barH_T} 4\barbeta_{t-1} \|x_t\|^2_{V^{-1}_{t-1}(\lam_T) } 
  \\&\le \cO( R^2d \log(T (1+ \lam_T^{-1}))) \cd \sum_{t\in\barH_T}  \|x_t\|^2_{V^{-1}_{t-1}(\lam_T) } 
  \\&=   \cO\del{ ( R^2d \log(T (1+ \lam_T^{-1})))\cd d \log(1 + \frac{T}{d\lam_T}) }
\end{align*}
where the last inequality is due to the standard elliptical potential lemma (e.g., \citet[Lemma 11]{ay11improved}).
It remains to bound $2S_*|H_T|$.
The elliptical potential count lemma (Lemma~\ref{lem:epc} in our appendix) implies that $|H_T| = \cO(d\log(1 + \frac{1}{\lam_T}))$.
Applying $|\barH_T| \le T$ to Eq.~\eqref{eq:proof-thm3-lastline} concludes the proof.

\section{THE FIXED ARM SET}
\label{sec:fixed}

We now consider the fixed arm set case, a special case of changing arm set where $\cX_t = \cX, \forall t\ge1,$ for some $\cX$.
In this setting, the algorithm knows $\cX$ from the beginning, which means that it is possible to plan in an informed way so that we can estimate the mean reward of each arm in a stable manner.
Motivated by this observation, we propose a new algorithm OLSOFUL (Ordinary Least Squares OFUL) that performs a two-step warmup procedure in order to establish good estimates of the mean rewards before starting an OFUL-like arm selection strategy.

The full pseudocode of OLSOFUL can be found in Algorithm~\ref{alg:olsoful}.
OLSOFUL first obtain a set of arms $\cX_0 := \{x_1,\ldots,x_{|\cX_0|}\}$ from KY sampling described in Algorithm~\ref{alg:warmup_psuedo}.
Let $\barV_t = \sum_{s=1}^t x_s x_s^\T$. 
KY sampling selects at most $2d$ arms from $\cX$ so that the determinant of the matrix $\barV_{|\cX_0|}$
is sufficiently large (details in Section~\ref{subsec:proof-olsoful}).
Let $\barV_{t} := \sum_{s=1}^{t} x_s x_s^\T$. 
The second warmup sampling simply samples arms that maximize $\|x\|^2_{\barV_{t-1}}$ until $\|x\|^2_{\barV_{t-1}} \le 1$ for every $x\in\cX$.
The second sampling is an auxiliary device that allows the regret analysis, although we speculate that it may not be necessary.
For the rest of the time steps, we perform the usual OFUL-like sampling but now with the MLE estimator $\hth_{t-1,0}$ computed by Eq.~\eqref{eq:mle} with $\lam=0$, which does not use a regularizer.
The confidence radius $\sqrt{\beta_{t-1}}$ is thus adjusted accordingly.
The difference is that we no longer have the usual $\sqrt{\lam}S_*$ term but rather a term of $\cO(\sqrt{d + \ln(1/\dt)})$, which makes it norm-agnostic.

\def\tw{{\mathsf{w}}}
\def\OLS{{\mathsf{OLS}}}
\begin{algorithm}[t]
  \caption{OLSOFUL}
  \label{alg:olsoful}
  \begin{algorithmic}
    \STATE \textbf{Input:} Arm set $\cX \subset \RR^d$
    \STATE Obtain $\cX_0:=\{x_1,\ldots, x_{|\cX_0|}\}$ by KY Sampling (Algorithm~\ref{alg:warmup}).
    \STATE $t \larrow |\cX_0| + 1$ 
    \WHILE {$\max_{x\in \cX}\|x\|_{\bar{V}_{t-1}^{-1}}^2>1$}
        \STATE $x_t = \arg \max_{x\in \cX} \|x\|^2_{\bar{V}_{t-1}^{-1}}$.
        \STATE Pull the arm $x_t$ and receive reward $y_t\in\RR$.
        \STATE $t \larrow t + 1$
    \ENDWHILE
    \STATE     $\tau \larrow t -1 $
    \FOR{$t=\tau+1,\dots$}
        \STATE Compute the MLE $\hat{\theta}_{t-1,0}$ by Eq.~\eqref{eq:mle} with $\lam=0$.
        \STATE ${\sqrt{\beta_{t-1}^{\OLS}}} := R \del{ \sqrt{\ln\del{\fr{|\barV_{t-1}|}{|\barV_{\tau}|\dt^2} } } \!+\! \sqrt{\ln(4) d + 4\ln(\fr1\dt)} }$
        \STATE Pull $x_t=\arg\max_{x\in\cX}\max_{\theta \in C_{t-1}} \langle x,\theta\rangle $ where 
         \begin{align*}
          {C_{t-1} }= \cbr{\th\in\RR^d: \| \hth_{t-1,0}-\th\|_{\barV_{t-1}} \le \blue{\sqrt{\beta^{\OLS}_{t-1}}}  }~.
        \end{align*}
        \STATE Receive reward $y_t\in\RR^d$~.
    \ENDFOR
  \end{algorithmic}
\end{algorithm}

OLSOFUL enjoys the following regret guarantee.
\begin{theorem}\label{thm:olsoful}
    Suppose we run \textup{OLSOFUL} (Algorithm~\ref{alg:olsoful}).
    Then, with probability at least $1-2\delta$,
    \begin{align*}
      \Reg_T = \cO\del{R d\sqrt{T\ln^2\del{\frac{T}{d}}} + S_* d \ln(d) }.
    \end{align*}
\end{theorem}
It is easy to see that the regret of OLSOFUL matches that of OFUL (Theorem~\ref{thm:oful}).
This implies that there is no cost of not knowing the norm of $\th^*$ in the worst-case regret scenario.
Note that KY sampling can be implemented so the overall time complexity is $\cO(d^2|\cX|)$, meaning that the time complexity of OLSOFUL is at most that of OFUL.
In this respect, we present a computationally efficient version of KY sampling in the appendix. 
Furthermore, OLSOFUL allows us not to regularize the intercept term, which is a standard practice in statistics and ML for supervised learning.
Thus, the rewards can be shifted by an arbitrary but fixed constant and the regret bound does not change.
Such a property is not available to OFUL and other existing algorithms, except for elimination algorithms~\cite[Section 21]{lattimore20bandit}.
While, elimination algorithms have superior regret bounds when $|\cX| = o(e^d)$, in practical regimes they have large regret due to its unique behavior of discarding samples.

\begin{algorithm}[t]
  \caption{BH Algorithm~\cite{betke93approximating,kumar05minimum}}
  \label{alg:warmup_psuedo}
  \begin{algorithmic}
    \STATE \textbf{Input:} the set of points $\cX \subset \RR^d$
    \STATE If $|\cX| \le 2d$, then $\cX_0 \larrow \cX$. Return.
    \STATE $\Psi \larrow{0}$, $\cX_0 \larrow \emptyset$, $i \larrow 0$.
    \WHILE{$\RR^d\setminus \Psi \neq \emptyset$}
    \STATE $i \larrow i + 1$
    \STATE Pick an arbitrary direction $b^i \in \RR^d$ in the orthogonal complement of $\Psi$.
    \STATE $p \larrow \arg \max_{x\in\cX} \la b^i, x \ra$; $\cX_0 \larrow \cX_0 \cup \{p\}$ 
    \STATE $q \larrow \arg \min_{x\in\cX} \la b^i, x \ra$; $\cX_0 \larrow \cX_0 \cup \{q\}$ 
    \STATE $\Psi \larrow \mathsf{Span}(\Psi, \{p - q\})$
    \ENDWHILE
  \end{algorithmic}
\end{algorithm} 

\begin{remark}
  One can replace our two-step warmup procedure with a design of experiment approach (specifically G-optimal design) and obtain a similar regret bound.
  That is, solve $\min_{\pi\in\Delta^{|\cX|}} \max_{x} \|x\|^2_{\barV(\pi)^{-1}}$ where $\barV(\pi)=\sum_{x\in\cX} \pi_x x x^\T$ and use a rounding procedure to select samples in a way that the fraction of samples spent on each arm is approximately $\pi_x$; see \citet[Appendix B]{fiez19sequential} for a discussion of rounding procedures.
  However, such an approach is computationally more expensive than ours as it requires iterative algorithms for solving the optimization problem.
  In contrast, our simple warm-up procedure consists of closed-form expressions and is computationally efficient.
  Alternatively, one can use the approximate barycentric spanner~\cite{awerbuch04adaptive} that consists of closed form operations but its time complexity is at least $\Omega(d^3 \log(d) |\cX|)$ compared to $\cO(d^2|\cX|)$ of our efficient implementation in the appendix.
\end{remark}

\subsection{Proof of Theorem~\ref{thm:olsoful}}
\label{subsec:proof-olsoful}

We now provide a sketch of the proof, leaving the full proof to the appendix.
The proof has two main parts.
First, we need to find an upper bound on $\tau$ so that the regret up to time $\tau$, which is the length of the warmup phase, can be bounded by $2S_* \tau$.
Second, we need to bound the regret after $\tau$ where the main novelty is our confidence set that replaces $|\lam I|$ in the confidence width $\sqrt{\beta_{t-1}}$ of OFUL with $|\barV_\tau|$.
We provide the main ideas only and leave the full proof to the appendix.
Let us first describe the property of the warmup procedure.
For $\pi\in\Delta^{|\cX|}$, define
\begin{align}
    \barV(\pi) = \sum_{x\in\cX} \pi_x  x x^\T~.
\end{align}

First, we claim that 
\begin{align*}
    \tau = \cO(d\ln(d))~.
\end{align*}
To see this, let $\pi^* = \arg\min_{\pi \in \Delta^{|\cX|}} \max_{x\in\cX}  \|x\|^2_{\barV(\pi)}$, the solution of the G-optimal design problem.
By~\citet{kiefer60theequivalence}, $\pi^* = \arg\max_{\pi \in \Delta^{|\cX|}} \max_{x\in\cX} |\barV(\pi)|$ as well.
This implies that $|\barV_\tau| = \tau^d|\fr{1}{\tau}\barV_\tau| \le \tau^d|\barV(\pi^*)|$.
Let $\om = |\cX_0|$ be the final time step spent on the first stage of the warmup.
Then, $|\barV_{\om}|  = \om^d |\fr1{\om} \barV_{\om}|$.
These bounds on the determinants together imply
\begin{align*}
  &\ln |\barV_\tau|  - \ln | \barV_\om | 
  \\&\le d\ln(\tau ) + \ln(|\barV(\pi^*)|)  -  \ln|\fr1\om \barV_\om| - d \ln (\om) 
  \\&\le d\ln(\tau ) + \cO(d\ln(d)) - d \ln (\om)
\end{align*}
where the last inequality is due to \citet[Theorem 4.2]{kumar05minimum}, meaning that the first stage of the warmup allocates arms so the resulting (normalized) determinant is sufficiently close to the ideal allocation of $\pi^*$.
On the other hand, 
\begin{align*}
  \ln |\barV_\tau|  - \ln | \barV_\om |
  &=\sum_{t=\om+1}^\tau \ln ( 1 + \|x_t\|^{2}_{\barV_{t-1}^{-1} }) 
  \\&> (\tau-\om) \ln(2) 
\end{align*}
where the last inequality is due to the definition of $\tau$.
Using $\om \le 2d$, one can solve the sandwich bound formed by the two displays above, which concludes the proof of the claim above.
This, in turn, proves that the warmup procedure results in $\cO(S_* d\ln(d))$ regret.

Let us now derive our confidence bound used after the warmup. 
Let $X_{\le\tau} \in \RR^{\tau \times d}$ ($X_{>\tau} \in \RR^{(t-\tau) \times d}$) be the design matrix consisting of the pulled arms during the warmup phase (after the warmup up time step $t$, respectively).
Define $\eta_{\le\tau}=(\eta_1, \ldots, \eta_\tau)^\T$ and $\eta_{>\tau} = (\eta_{\tau+1},\ldots, \eta_t)^{\T}$.
We first note that the MLE $\hth_{t,0}$ satisfies
\begin{align*}
  \barV_t(\hth_{t,0} - \th^* ) = X_{>\tau} ^\T \eta_{>\tau} - \lam \barV_\tau (\th^* - \hth_{\tau,0})
\end{align*}
Then,
\begin{align*}
  &x^\T(\hth_{t,0}  - \th^*) 
  \\&= x^\T \barV_t^{-1} X^{\T}_{>\tau}\eta_{>\tau} - \lam x^\T \barV_t^{-1}\barV_{\tau}(\th^* - \hth_{\tau,0})  
\end{align*}
With probability at least $1-\dt$, we can bound the first term by $\|x\|_{\barV_t^{-1}} \| X^{\T}_{>\tau}\eta_{>\tau}\|_{\barV_t^{-1}} \le \|x\|_{\barV_t^{-1}} R\sqrt{\ln(\fr{|\barV_t|}{|\barV_\tau|\dt^2} )}$ for all $t\ge\tau+1$, which is formally stated in the appendix.
Such a result is a consequence of applying the standard self-normalized confidence bound of~\citet[Theorem 1]{ay11improved} where the parameter $V$ therein for the prior is set as $\barV_\tau$.
Furthermore, with probability at least $1-\dt$, one can bound
\begin{align*}
 - x^\T \barV_t^{-1}\barV_{\tau}(\th^* - \hth_{\tau,0})
  &=  x^\T \barV_t^{-1}X_{\le \tau}^\T \eta_{\le\tau}
\\&\le \|x\|_{\barV_{t}^{-1}} \|X^\T_{\le\tau} \eta_{\le\tau}  \|_{\barV_\tau^{-1}} 
\\&\le R\sqrt{2\ln(2) d + 4\ln(1/\dt)}
\end{align*}
where the last inequality is a consequence of applying~\citet[Theorem 1]{ay11improved} in a similar ways as above.

To bound the regret after $\tau$, the standard analysis goes through.
The only part that remains is that the usual term $\ln(|V_t(\lam)|/|\lam I|)$ is replaced by $\ln(|\barV_t|/|\barV_\tau|)$, which can be bounded by $\cO(d\ln(t/d))$ using the fact that $|\barV_\tau|$ is sufficiently large.

\def\OFULzero{{OFUL$_0$}\xspace}

\section{EXPERIMENTS}
\label{sec:expr}

\begin{figure*}
  \centering
  \begin{tabular}{ccc}
      \includegraphics[width=0.31\linewidth]{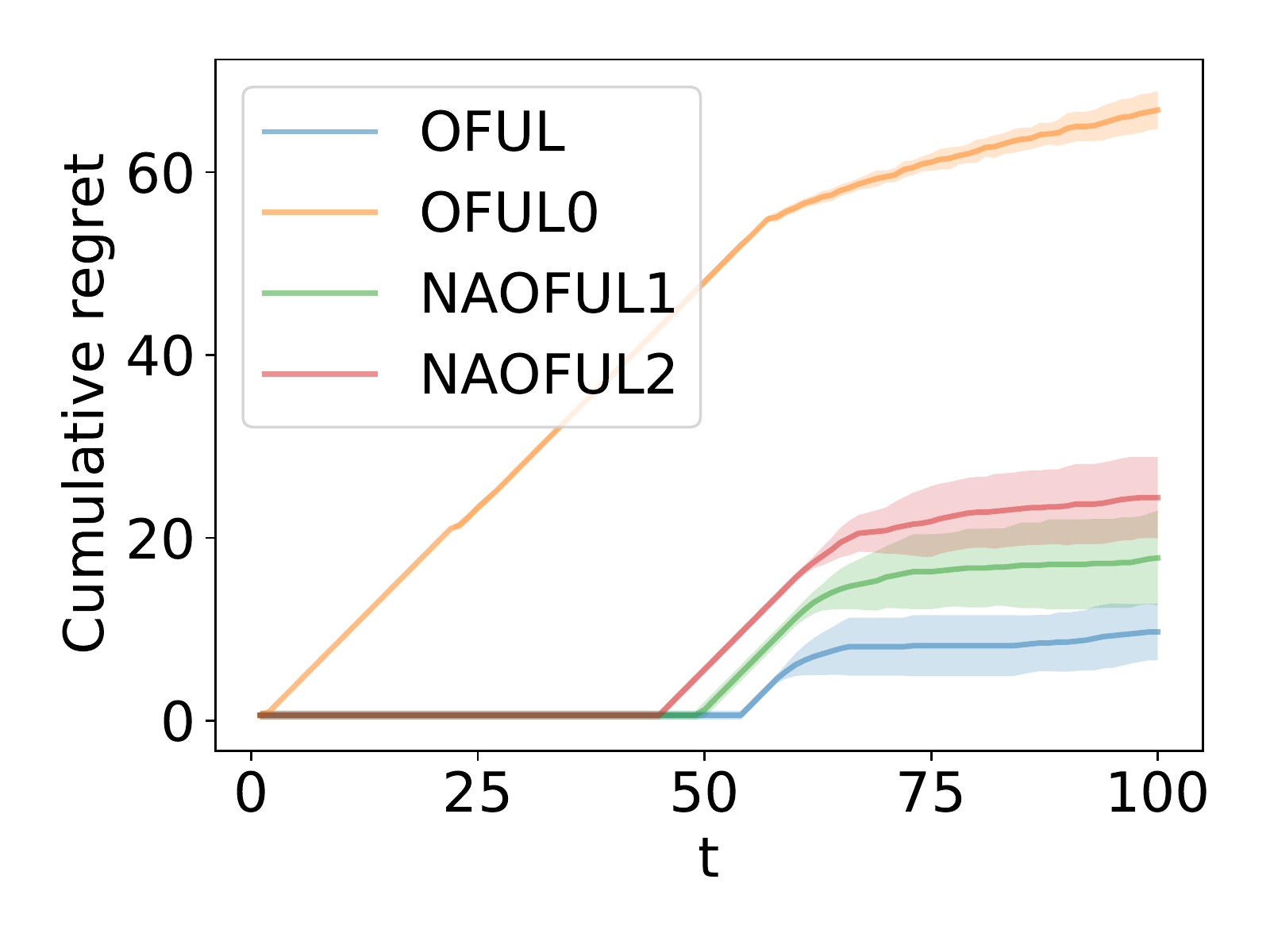}
     &\includegraphics[width=0.31\linewidth]{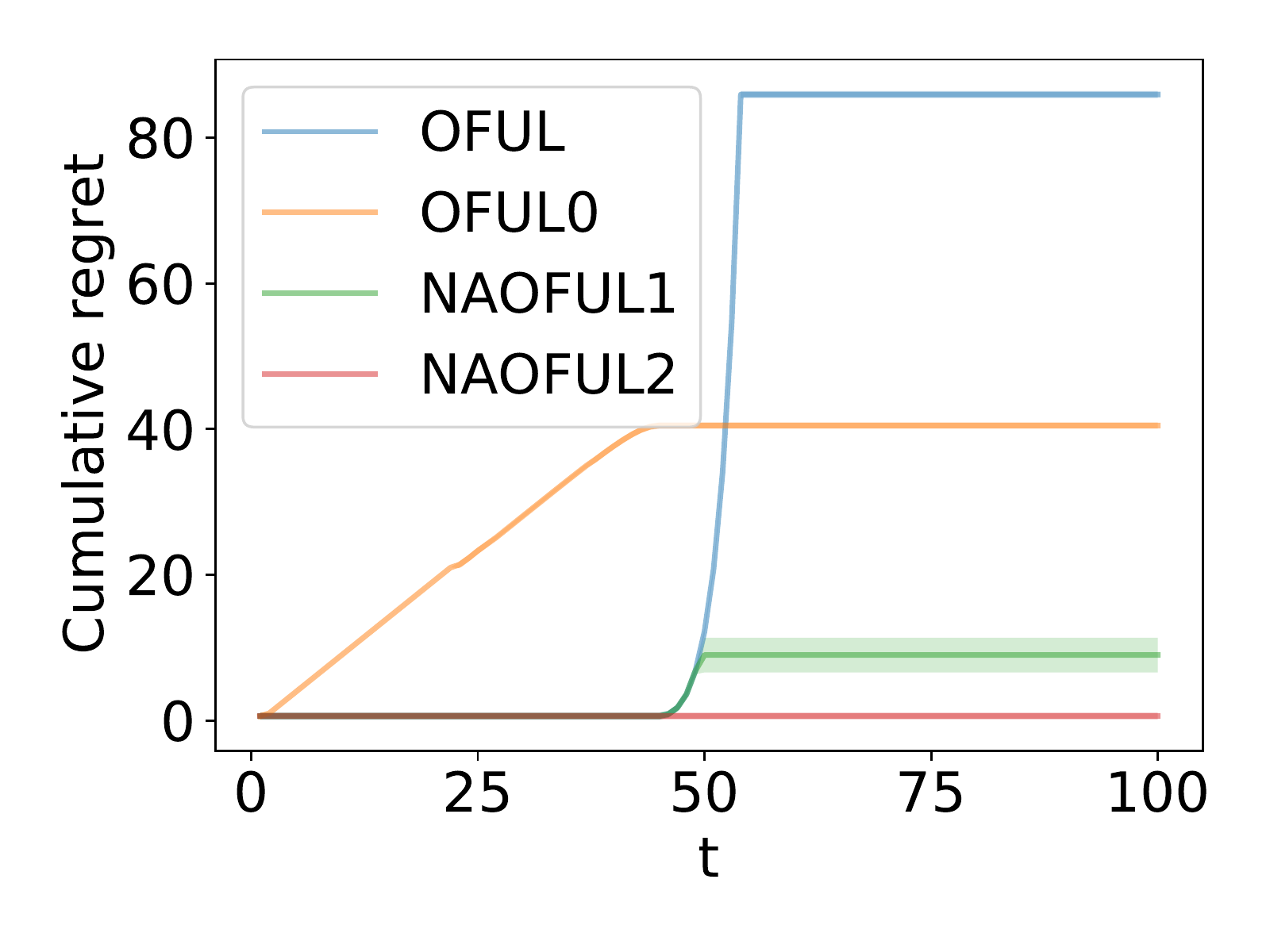} 
     &\includegraphics[width=0.31\linewidth]{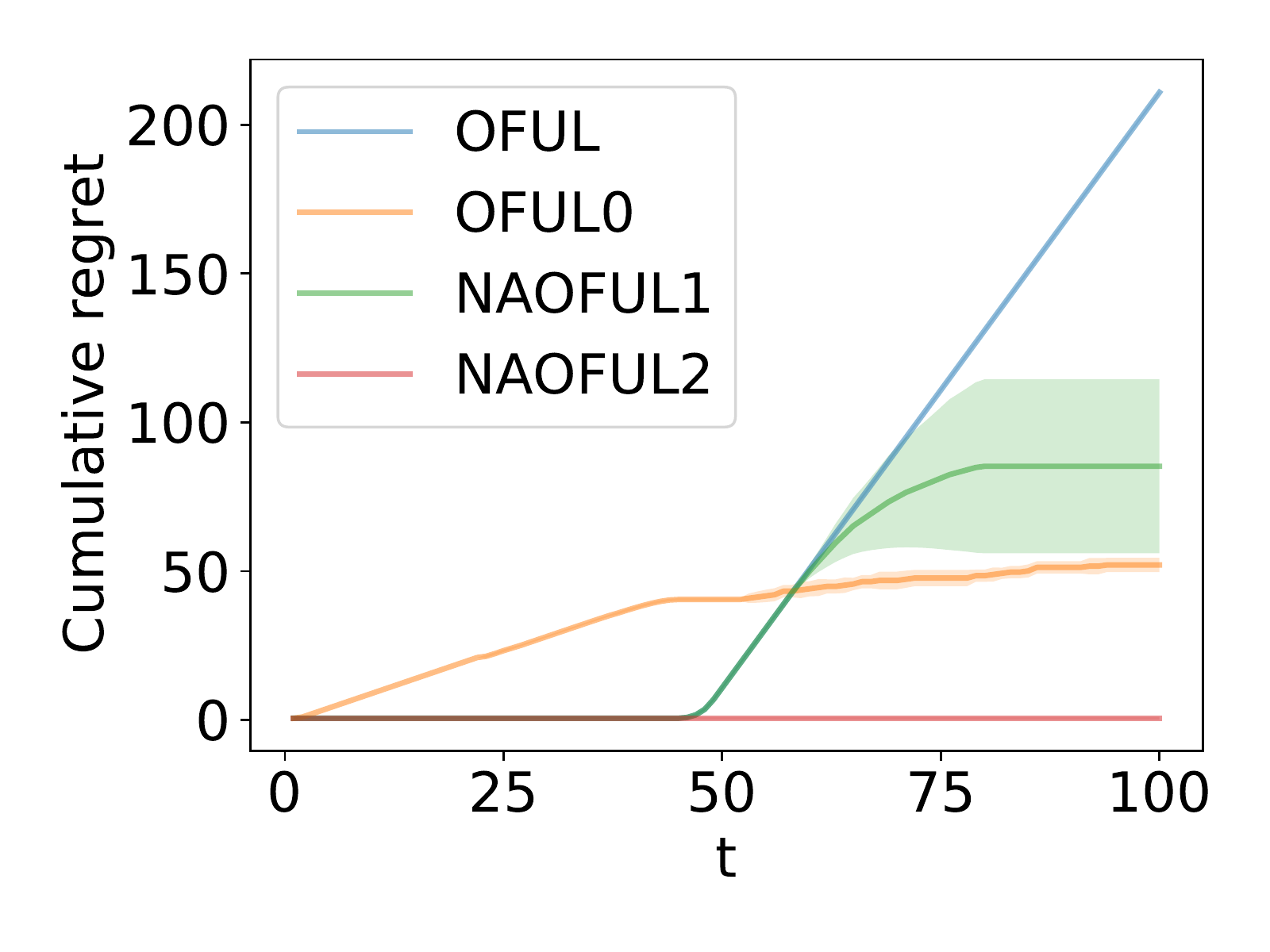} 
     \\ (a) & (b) & (c)
  \end{tabular}
  \caption{Experimental results. The plots (a), (b), and (c) corresponds to the instance  (a), (b), and (c) in Table~\ref{tab:expr}. Our methods consistently perform the best or near best in all the environments whereas OFUL and \OFULzero can suffers an almost linear regret in some of the instances.}
  \label{fig:expr}
\end{figure*}

In this section, validate the need for norm-agnostic algorithms and their effectiveness via a numerical evaluation in synthetic environments.

\textbf{Algorithms.}
We test three algorithms: OFUL (Algorithm~\ref{alg:oful} with $\lambda=1$ and $S=1$), NAOFUL (Algorithm~\ref{alg:naoful}, and \OFULzero (described below).
We run two versions of NAOFUL with $\lam_t = R^2 d / t^\alpha$: NAOFUL1 for $\alpha=1$ and NAOFUL2 for $\alpha=2$.
\OFULzero is OFUL without a regularization.
\OFULzero is a reasonable attempt to be norm-agnostic.
Specifically, we use the confidence set from \citet[Eq. (20.3)]{lattimore20bandit} for the ordinary least squares estimator $\hth_{t,0}$ that is the solution of Eq.~\eqref{eq:mle} with $\lam=0$ while breaking ties by taking the minimum norm solution:
\begin{align}
  \PP\Bigg(\exists t\ge1:  &\|\hth_{t,0} - \th^*\|_{V_t(0)}  \notag
  \\ &\ge R\sqrt{8\del{d\log(6) + \log\del[0]{\frac{\pi^2t^2}{6\dt}}} }  \Bigg)
  \le \dt~. \label{eq:oful0}
\end{align}
We then construct a confidence set with the constraint above to choose the optimistic arm as in Eq.~\eqref{eq:optimism}.
Such a confidence set does not involve $S_*$ and thus is norm-agnostic like NAOFUL, but \OFULzero does not have a known regret bound.\footnote{
  In fact, the proof of Eq.~\eqref{eq:oful0} requires that the chosen arms $\{x_1,\ldots,x_t\}$ do not depend on $\{y_1, \ldots, y_t\}$, which is violated in \OFULzero. 
  While this is not ideal, \OFULzero serves our purpose because \OFULzero's confidence bounds were never wrong in our experiments.
  Alternatively, one could use OFUL with an extremely small $\lambda$ (e.g., $10^{-16}$) and obtain similar results, but its estimator tends to be numerically unstable.
}

\textbf{Synthetic data.}
To evaluate the performance of the norm-agnostic algorithm for the changing arm set setting, we consider three problem instances.
We will consider a common arm set
\begin{align*}
\cX_t := \{(1,0)^\T, (0, \min\cbr{u, a_0 \cd b^t } )^\T\}, \forall t    ~.
\end{align*}
When arm $x_t$ is pulled, the rewards are generated by $y_t = x_t^\T\th^* + \eta_t$ where $\eta_t \sim \cN(0,1)$.
For ease of exposition, we refer to the arm $(1,0)^\T$ as A1 and $(0, \min\cbr{u, a_0 \cd b^t } )^\T$ as A2.
We describe the choices of the instance parameters in Table~\ref{tab:expr}.
\begin{table}
  \centering
  {\small
  \begin{tabular}{cccccc}\hline
    Inst. & $a_0$ & $b$ & $u$ & $\th^*$ & Best arm
    \\\hline   (a)  & $10^{-10}$ & $1.5$ & $1.0$ & $(1,0)^\T$ & $(1,0)^\T$
    \\ (b)  & $10^{-10}$ & $1.5$ & $1.0$ & $(1,100)^\T$ & $(1,0)^\T$ if $t\le45$
    \\ (c)  & $10^{-10}$ & $1.5$ & $0.05$ & $(1,100)^\T$ & $(1,0)^\T$ if $t\le45$
    \\\hline
  \end{tabular}
  }
  \caption{Our synthetic problem instances.}
  \label{tab:expr}
\end{table}
The intention for instance $(a)$ in Table~\ref{tab:expr}, we have $\|S_*\|  \le 1$, so OFUL's regret bound is valid, OFUL should work well whereas for instances $(b)$ and $(c)$ we have $\|S_*\| \gg 1$, so OFUL is expected to perform badly.
We run all the algorithms up to $T=100$ on each instance and measure their regret, which we repeat 10 times to obtain mean and standard deviation.

\textbf{Results.}
We plot the mean and standard deviation of the regret of each algorithm in Figure~\ref{fig:expr}, which shows that OFUL and \OFULzero have almost linear regret for some instances whereas NAOFUL1 and NAOFUL2 maintain low regret for all.
Specifically, for instance (a), one can see that \OFULzero is significantly worse than the rest, showing an almost linear regret. 
In fact, given $T$, one can construct an instance where \OFULzero suffers a linear regret up to time $T$ by setting $a_0$ arbitrarily small.
This is because its upper confidence bound (UCB), though valid, does not fall sufficiently low to allow the algorithm to pull the true best arm due to the fact that the second dimension of A2 is increasing geometrically over time. 
In contrast, NAOFUL1, NAOFUL2, and OFUL are able to ignore A2 initially but start to pull it later in order to avoid the risk of A2 becoming the best arm.
As noted in Table~\ref{tab:expr}, instance (b) switches the best arm at $t=46$. 
This forces the regret of \OFULzero, which mostly pulls A2, to stop growing, and the regret of OFUL, whose UCB on A2 is underestimated, to start exploding.
In contrast, both NAOFUL1 and NAOFUL2 are able to quickly start pulling A2 when the best arm switches, and NAOFUL2 is particularly good at it.
Note that OFUL, though late, is still able to pull A2 and stop growing its regret.
However, the results for instance (c) shows that by adjusting $u$, one can force OFUL to have linear regret; what happens is that the second dimension of A2 is just large enough to make it the best arm but not large enough to force OFUL's UCB of A2 to be the UCB of A1. 
\begin{remark}
  Readers may wonder why the regret plot for NAOFUL does not seem to follow a $\sqrt{T}$ growth rate. 
  The regret bound we show in this paper is the \textit{worst-case} regret, so individual regret curves for each instance may behave differently. 
  For instance-optimal regret bounds, we refer to~\citet{lattimore17theend}.
\end{remark}

\section{CONCLUSION}
\label{sec:conclusion}

In this paper, we have removed the common assumption in linear bandits that we have access to an upper bound $S$ on the norm $\|\th^*\|$ and/or the range of the mean rewards.
We believe our work has implications for problems beyond bandits because the standard linear bandit algorithm like OFUL or similar versions are abundant in online selective sampling~\cite{Dekel10robust,dekel12selective,gentile14onmultilabel} and reinforcement learning~\cite{zhou21nearly,du21bilinear}.

Our work opens up numerous future research directions.
First, we believe that our regret upper bound for the changing arm set setting might be loose.
It would be interesting to identify the optimal regret for norm-agnostic algorithms by finding a matching lower and upper bound on the regret.
Second, extending our work to kernels would be interesting.
Kernels are nontrivial to apply our techniques since their regret bounds scale with the effective dimension that in turn depends on the regularizer.
Thus, using a decreasing regularizer would increase the effective dimension over time, so a careful or even data-dependent scheduling of $\{\lam_t\}$ might be needed in order to keep the regret close to existing bounds obtained with knowledge of $S$.
Next, it would be interesting to also remove the assumption that the arm vector's norm is bounded by $1$.
Note that this is not a problem for the fixed arm set setting since one can scale the given arm set and scale their norm to be less than or equal to 1. 
The interesting case is the changing arm set setting where the norm of the arms may grow indefinitely over time.  
Finally, we believe a more extensive empirical evaluation of our algorithms on a number of differing real world data sets would help inform the risk of misspecifying $S$. 


\subsubsection*{Acknowledgements}

Sunder Sethuraman was supported by ARO W911NF-181-0311.
Kwang-Sung Jun was supported by Data Science Academy and Research, Innovation \& Impact at University of Arizona.


%
%

\subsection*{References}

\bibliography{library-shared}

\begin{thebibliography}{8}
\providecommand{\natexlab}[1]{#1}
\providecommand{\url}[1]{\texttt{#1}}
\expandafter\ifx\csname urlstyle\endcsname\relax
  \providecommand{\doi}[1]{doi: #1}\else
  \providecommand{\doi}{doi: \begingroup \urlstyle{rm}\Url}\fi

\bibitem[Abbasi-Yadkori et~al.(2011)Abbasi-Yadkori, Pal, and
  Szepesvari]{ay11improved}
Yasin Abbasi-Yadkori, David Pal, and Csaba Szepesvari.
\newblock {Improved Algorithms for Linear Stochastic Bandits}.
\newblock In \emph{Advances in Neural Information Processing Systems
  (NeurIPS)}, pages 1--19, 2011.

\bibitem[Antos et~al.(2010)Antos, Grover, and Szepesv{\'{a}}ri]{antos10active}
Andr{\'{a}}s Antos, Varun Grover, and Csaba Szepesv{\'{a}}ri.
\newblock {Active learning in heteroscedastic noise}.
\newblock \emph{Theoretical Computer Science}, 411\penalty0 (29-30):\penalty0
  2712--2728, 2010.

\bibitem[Betke and Henk(1993)]{betke93approximating}
Ulrich Betke and Martin Henk.
\newblock {Approximating the volume of convex bodies}.
\newblock \emph{Discrete {\&} Computational Geometry}, 10\penalty0
  (1):\penalty0 15--21, 1993.

\bibitem[Kiefer and Wolfowitz(1960)]{kiefer60theequivalence}
Jack Kiefer and Jacob Wolfowitz.
\newblock {The equivalence of two extremum problems}.
\newblock \emph{Canadian Journal of Mathematics}, 12:\penalty0 363--366, 1960.

\bibitem[Kim et~al.(2021)Kim, Yang, and Jun]{kim21improved}
Yeoneung Kim, Insoon Yang, and Kwang-Sung Jun.
\newblock {Improved Regret Analysis for Variance-Adaptive Linear Bandits and
  Horizon-Free Linear Mixture MDPs}, 2021.

\bibitem[Kumar and Yildirim(2005)]{kumar05minimum}
Piyush Kumar and E~Alper Yildirim.
\newblock {Minimum-volume enclosing ellipsoids and core sets}.
\newblock \emph{Journal of Optimization Theory and applications}, 126\penalty0
  (1):\penalty0 1--21, 2005.

\bibitem[Lattimore and Szepesv{\'{a}}ri(2020)]{lattimore20bandit}
Tor Lattimore and Csaba Szepesv{\'{a}}ri.
\newblock \emph{{Bandit Algorithms}}.
\newblock Cambridge University Press, 2020.

\bibitem[Russo and {Van Roy}(2013)]{russo13eluder}
Daniel Russo and Benjamin {Van Roy}.
\newblock {Eluder dimension and the sample complexity of optimistic
  exploration}.
\newblock In \emph{Advances in Neural Information Processing Systems
  (NeurIPS)}, pages 2256--2264, 2013.

\end{thebibliography}


\begin{thebibliography}{28}
\providecommand{\natexlab}[1]{#1}
\providecommand{\url}[1]{\texttt{#1}}
\expandafter\ifx\csname urlstyle\endcsname\relax
  \providecommand{\doi}[1]{doi: #1}\else
  \providecommand{\doi}{doi: \begingroup \urlstyle{rm}\Url}\fi

\bibitem[Abbasi-Yadkori et~al.(2011)Abbasi-Yadkori, Pal, and
  Szepesvari]{ay11improved}
Y.~Abbasi-Yadkori, D.~Pal, and C.~Szepesvari.
\newblock {Improved Algorithms for Linear Stochastic Bandits}.
\newblock In \emph{Advances in Neural Information Processing Systems
  (NeurIPS)}, pages 1--19, 2011.

\bibitem[Abe and Long(1999)]{abe99associative}
N.~Abe and P.~M. Long.
\newblock {Associative reinforcement learning using linear probabilistic
  concepts}.
\newblock In \emph{Proceedings of the International Conference on Machine
  Learning (ICML)}, pages 3--11, 1999.

\bibitem[Antos et~al.(2010)Antos, Grover, and Szepesv{\'{a}}ri]{antos10active}
A.~Antos, V.~Grover, and C.~Szepesv{\'{a}}ri.
\newblock {Active learning in heteroscedastic noise}.
\newblock \emph{Theoretical Computer Science}, 411\penalty0 (29-30):\penalty0
  2712--2728, 2010.

\bibitem[Auer(2002)]{auer02using}
P.~Auer.
\newblock {Using Confidence Bounds for Exploitation-Exploration Trade-offs}.
\newblock \emph{Journal of Machine Learning Research}, 3:\penalty0 397--422,
  2002.

\bibitem[Awerbuch and Kleinberg(2004)]{awerbuch04adaptive}
B.~Awerbuch and R.~D. Kleinberg.
\newblock {Adaptive Routing with End-to-End Feedback: Distributed Learning and
  Geometric Approaches}.
\newblock In \emph{Proceedings of the ACM Symposium on Theory of Computing
  (STOC)}, pages 45--53, 2004.

\bibitem[Betke and Henk(1993)]{betke93approximating}
U.~Betke and M.~Henk.
\newblock {Approximating the volume of convex bodies}.
\newblock \emph{Discrete {\&} Computational Geometry}, 10\penalty0
  (1):\penalty0 15--21, 1993.

\bibitem[Dani et~al.(2008)Dani, Hayes, and Kakade]{dani08stochastic}
V.~Dani, T.~P. Hayes, and S.~M. Kakade.
\newblock {Stochastic Linear Optimization under Bandit Feedback.}
\newblock In \emph{Proceedings of the Conference on Learning Theory (COLT)},
  pages 355--366, 2008.

\bibitem[Dekel et~al.(2010)Dekel, Gentile, and Sridharan]{Dekel10robust}
O.~Dekel, C.~Gentile, and K.~Sridharan.
\newblock {Robust selective sampling from single and multiple teachers}.
\newblock In \emph{Proceedings of the Conference on Learning Theory (COLT)},
  2010.

\bibitem[Dekel et~al.(2012)Dekel, Gentile, and Sridharan]{dekel12selective}
O.~Dekel, C.~Gentile, and K.~Sridharan.
\newblock {Selective sampling and active learning from single and multiple
  teachers}.
\newblock \emph{Journal of Machine Learning Research}, 13:\penalty0 2655--2697,
  2012.

\bibitem[Du et~al.(2021)Du, Kakade, Lee, Lovett, Mahajan, Sun, and
  Wang]{du21bilinear}
S.~S. Du, S.~M. Kakade, J.~D. Lee, S.~Lovett, G.~Mahajan, W.~Sun, and R.~Wang.
\newblock {Bilinear Classes: A Structural Framework for Provable Generalization
  in RL}.
\newblock \emph{arXiv preprint arXiv:2103.10897}, 2021.

\bibitem[Fiez et~al.(2019)Fiez, Jain, Jamieson, and Ratliff]{fiez19sequential}
T.~Fiez, L.~Jain, K.~G. Jamieson, and L.~Ratliff.
\newblock {Sequential Experimental Design for Transductive Linear Bandits}.
\newblock In \emph{Advances in Neural Information Processing Systems
  (NeurIPS)}, volume~32, 2019.

\bibitem[Gentile and Orabona(2014)]{gentile14onmultilabel}
C.~Gentile and F.~Orabona.
\newblock {On Multilabel Classification and Ranking with Bandit Feedback}.
\newblock \emph{Journal of Machine Learning Research}, 15:\penalty0 2451--2487,
  2014.

\bibitem[Ghosh et~al.(2021)Ghosh, Sankararaman, and Kannan]{ghosh2021problem}
A.~Ghosh, A.~Sankararaman, and R.~Kannan.
\newblock Problem-complexity adaptive model selection for stochastic linear
  bandits.
\newblock In \emph{International Conference on Artificial Intelligence and
  Statistics (AISTATS)}, pages 1396--1404. PMLR, 2021.

\bibitem[Kiefer and Wolfowitz(1960)]{kiefer60theequivalence}
J.~Kiefer and J.~Wolfowitz.
\newblock {The equivalence of two extremum problems}.
\newblock \emph{Canadian Journal of Mathematics}, 12:\penalty0 363--366, 1960.

\bibitem[Kim et~al.(2021)Kim, Yang, and Jun]{kim21improved}
Y.~Kim, I.~Yang, and K.-S. Jun.
\newblock {Improved Regret Analysis for Variance-Adaptive Linear Bandits and
  Horizon-Free Linear Mixture MDPs}, 2021.

\bibitem[Koc{\'{a}}k et~al.(2020)Koc{\'{a}}k, Munos, Kveton, Agrawal, and
  Valko]{kocak20spectral}
T.~Koc{\'{a}}k, R.~Munos, B.~Kveton, S.~Agrawal, and M.~Valko.
\newblock {Spectral bandits}.
\newblock \emph{Journal of Machine Learning Research}, 2020.

\bibitem[Kumar and Yildirim(2005)]{kumar05minimum}
P.~Kumar and E.~A. Yildirim.
\newblock {Minimum-volume enclosing ellipsoids and core sets}.
\newblock \emph{Journal of Optimization Theory and applications}, 126\penalty0
  (1):\penalty0 1--21, 2005.

\bibitem[Lattimore and Szepesv{\'{a}}ri(2017)]{lattimore17theend}
T.~Lattimore and C.~Szepesv{\'{a}}ri.
\newblock {The end of optimism? An asymptotic analysis of finite-armed linear
  bandits}.
\newblock In \emph{Proceedings of the International Conference on Artificial
  Intelligence and Statistics (AISTATS)}, 2017.

\bibitem[Lattimore and Szepesv{\'{a}}ri(2020)]{lattimore20bandit}
T.~Lattimore and C.~Szepesv{\'{a}}ri.
\newblock \emph{{Bandit Algorithms}}.
\newblock Cambridge University Press, 2020.

\bibitem[Li et~al.(2010)Li, Chu, Langford, and Schapire]{li10acontextual}
L.~Li, W.~Chu, J.~Langford, and R.~E. Schapire.
\newblock {A Contextual-Bandit Approach to Personalized News Article
  Recommendation}.
\newblock \emph{Proceedings of the International Conference on World Wide Web
  (WWW)}, pages 661--670, 2010.

\bibitem[Orabona and Cesa-Bianchi(2011)]{orabona11better}
F.~Orabona and N.~Cesa-Bianchi.
\newblock {Better Algorithms for Selective Sampling}.
\newblock In \emph{Proceedings of the International Conference on International
  Conference on Machine Learning (ICML)}, pages 433--440, 2011.

\bibitem[Russo and Roy(2014)]{russo14learning}
D.~Russo and B.~V. Roy.
\newblock {Learning to Optimize via Posterior Sampling}.
\newblock \emph{Mathematics of Operations Research}, 39\penalty0 (4):\penalty0
  1221--1243, 2014.

\bibitem[Russo and {Van Roy}(2013)]{russo13eluder}
D.~Russo and B.~{Van Roy}.
\newblock {Eluder dimension and the sample complexity of optimistic
  exploration}.
\newblock In \emph{Advances in Neural Information Processing Systems
  (NeurIPS)}, pages 2256--2264, 2013.

\bibitem[Thompson(1933)]{thompson33onthelikelihood}
W.~R. Thompson.
\newblock {On the Likelihood that One Unknown Probability Exceeds Another in
  View of the Evidence of Two Samples}.
\newblock \emph{Biometrika}, 25\penalty0 (3/4):\penalty0 285, 1933.

\bibitem[Wang et~al.(2021)Wang, Awasthi, Dann, Sekhari, and
  Gentile]{wang2021neural}
Z.~Wang, P.~Awasthi, C.~Dann, A.~Sekhari, and C.~Gentile.
\newblock Neural active learning with performance guarantees.
\newblock In A.~Beygelzimer, Y.~Dauphin, P.~Liang, and J.~W. Vaughan, editors,
  \emph{Advances in Neural Information Processing Systems (NeurIPS)}, 2021.

\bibitem[Zhou et~al.(2021)Zhou, Gu, and Szepesvari]{zhou21nearly}
D.~Zhou, Q.~Gu, and C.~Szepesvari.
\newblock {Nearly minimax optimal reinforcement learning for linear mixture
  markov decision processes}.
\newblock In \emph{Proceedings of the Conference on Learning Theory (COLT)},
  pages 4532--4576. PMLR, 2021.

\bibitem[Zhu and Nowak(2020)]{zhu20regret}
Y.~Zhu and R.~Nowak.
\newblock {On regret with multiple best arms}.
\newblock \emph{Advances in Neural Information Processing Systems (NeurIPS)},
  33:\penalty0 9050--9060, 2020.

\bibitem[Zhu and Nowak(2021)]{zhu21pareto}
Y.~Zhu and R.~Nowak.
\newblock {Pareto Optimal Model Selection in Linear Bandits}.
\newblock \emph{arXiv preprint arXiv:2102.06593}, 2021.

\end{thebibliography}

\clearpage
\appendix
\thispagestyle{empty}
\onecolumn 
\aistatstitle{Supplementary Material: \\ Norm-Agnostic Linear Bandits }


%
%
%
%
\renewcommand \thepart{}
\renewcommand \partname{}
\part{} 
\vspace{-4em}
\parttoc 


\def\tw{{\mathsf{w}}}

\section{Minor error of~\citet[Theorem 3]{ay11improved}}
\label{sec:error}

We claim that the upper bound reported in Theorem~\ref{thm:oful} should have an additive lower order term $\Omega(d S^*)$.
To see this, imagine the arm set is $\cX_t = \{e_1,\ldots, e_d\}$ where $e_i$ is the $i$-th indicator vector.
Suppose $\theta^* = S^* \cdot e_{d}$.
This way, pulling a suboptimal arm incurs instantaneous regret of $S^*$.
Assume that OFUL breaks the tie chooses an arm $i\neq i^*$, which is an allowed behavior for OFUL.
Assume that there is no noise; i.e., $R=0$, which makes $\sqrt{\beta_{t-1}} = \sqrt{\lambda} S^*$.
Quick calculation shows that the UCB for the pulled arm becomes $\sqrt{\beta_{t-1}} \frac{1}{1+\lambda}$ and that for the unpulled arm becomes $\sqrt{\beta_{t-1}} \frac{1}{\lambda} $.
So, OFUL will play $x_t = e_t$ up to time step $d$, incurring the regret of $S^*(d-1)$.

The source of such a mismatch comes from a minor error in the proof of OFUL regret bound.
In the proof of Theorem 3 of~\citet{ay11improved}, they show that
\begin{align*}
    r_t 
    \le 2 \min\{\sqrt{\beta_{t-1}(\dt)} \norm{X}_{\bar V_{t-1}^{-1}} , 1\}
    \le 2\sqrt{\beta_{t-1}(\dt)} \min\{ \norm{X}_{\bar V_{t-1}^{-1}} , 1\}
\end{align*}
The second inequality is not true in general and becomes true only if $\beta_{t-1}(\dt) \ge 1$.
The analysis of OFUL in~\cite{lattimore20bandit} goes around the issue by assuming $R=1$.

\section{Proof of Theorem~\ref{thm:naoful}}
Define the instantaneous regret as $r_t := \max_{x\in\cX_t} \la x, \th^*\ra - \la x_t, \th^*\ra$, then 

\begin{proof}
Consider the following,
\begin{align*}
    \Reg_T 
     &=\sum_{t=1}^T r_t\\
     &= \sum_{t=1}^T\max_{x\in\cX_t} \la x, \th^*\ra - \la x_t, \th^*\ra\\
    & = \sum_{t=1}^{T} \one\{\lambda_t^{1/2}S_*> R\sqrt{d}\}r_t + \sum_{t=1}^{T} \one\{\lambda_t^{1/2}S_*\leq R\sqrt{d}\}r_t\\
    &=\sum_{t=1}^{t_0-1} r_t + \sum_{t=t_0}^{T} r_t
\end{align*}
where $t_0=\min\{t\geq 1: \lambda_t^{1/2}S\leq R\sqrt{d}\}$. Therefore consider the case $\lambda_t^{1/2}S_*> R\sqrt{d},$
\begin{align*}
    \sum_{t=1}^{T} \one\{\lambda_t^{1/2}S_*> R\sqrt{d}\} r_t &= \sum_{t=1}^{t_0-1} r_t\\
    &= \sum_{t=1}^{t_0-1} \max_{x\in\cX_t} \langle x-x_t,\th^*\rangle \\
    & \leq \sum_{t=1}^{t_0-1} \max_{x\in\cX_t} \|x-x_t\|_2\|\th^*\|_2\\
    & \leq \sum_{t=1}^{t_0-1} 2S_*\\
    & = 2(t_0-1)S_*.
\end{align*}

 Now consider the second term for $t\geq t_0$. Since $\lambda_t^{1/2}S_*\leq R\sqrt{d},$  we can adapt the analysis of Theorem 3 from \cite{ay11improved}. Adapting Theorem 2 of \cite{ay11improved} for our problem we have the following 
 
\begin{theorem} \label{Thm: Confidence_ell}
Define $t_0:=\min\{t\geq 1: \lambda_t^{1/2}S\leq R\sqrt{d}\}$, and recall $y_t=\langle x_t,\th^*\rangle+\eta_t$. Then, for any $\delta>0$ with $\delta_t = \frac{1}{t^2}\frac{6}{\pi^2}\delta$, we have, with probability at least $1-\delta$,
\[\forall t \ge t_0,\  \th^* \in \barC_t =\left\{\th^*\in\mathbb{R}^d:\ \|\hth_t-\th^*\|_{V_t(\lambda_t)} \leq R\sqrt{2\log(\frac{|V_t(\lambda_t)|^{1/2}|\lambda_t I|^{-1/2}}{\delta_t})}+R\sqrt{d}\right\}\]
where $\hat{\theta}_t= \arg\min_{\th}\sum_{k=1}^t(y_k-\langle x_k, \th\rangle)^2+{\lambda_{t}}\|\theta\|_2^2$. 
\end{theorem}
\begin{proof}
We know from Theorem 3 from \cite{ay11improved} that for any fixed $\lambda_t$ and $\delta_t=\frac{1}{t^2}\frac{6}{\pi^2}>0$ that with probability at least $1-\dt$, for all $t\geq t_0$ 
\[\th^* \in \barC_t =\left\{\th^*\in\mathbb{R}^d:\ \|\hat{\theta}_t-\th^*\|_{V_t(\lambda_t)} \leq R\sqrt{2\log(\frac{|V_t(\lambda_t)|^{1/2}|\lambda_t I|^{-1/2}}{\delta_t})}+R\sqrt{d}\right\}.\]
Define $E =\cap_{t=t_0}^T \barC_t$. We will show that with probability at least $1-\delta$ and for any $t\geq t_0$ and for a decreasing sequence of $\lambda_t$ and $\delta_t= \frac{1}{t^2}\frac{6}{\pi^2}\delta$ that $P(E)\geq 1- \delta$. Indeed, 
\begin{align*}
    P(E^c)&= P( \cup_{t=t_0}^T {\barC_t}^c)\\
    & \leq \sum_{t=t_0}^T P({\barC_t}^c) \tag{By the Union Bound}\\
    & \leq \sum_{t=t_0}^T \delta_t\\
    & \leq \sum_{t=t_0}^T \frac{1}{t^2}\frac{6}{\pi^2}\delta\\
    & \leq \delta \tag{Since $\sum_{i=1}^\infty \frac{1}{i^2}=\frac{\pi^2}{6}$}
\end{align*}

\end{proof}
Therefore, with $\lambda_t^{1/2}S_*\leq R\sqrt{d}$, using Theorem \ref{Thm: Confidence_ell} where $\sqrt{\bar{\beta}_{t-1}(\delta)}= R\sqrt{2\log(\frac{|V_t(\lambda_t)|^{1/2}|\lambda_t I|^{-1/2}}{\delta_t})}+R\sqrt{d}$
\begin{align*}
    r_t & = \max_{x\in\cX_t} \la x, \th^*\ra - \la x_t, \th^*\ra\\
    & \leq \langle x_t,\hth_t\rangle+ \sqrt{\bar{\beta}_{t-1}(\delta)}\|x_t\|_{V_{t-1}^{-1}(\lambda_t)} -\langle x_t,\th^*\rangle\\
    & \sr{(a)}{\le} \langle x_t,\hth_t\rangle+ \sqrt{\bar{\beta}_{t-1}(\delta)}\|x_t\|_{V_{t-1}^{-1}(\lambda_t)} -\langle x_t,
    \hat{\theta}_t\rangle +\sqrt{\bar{\beta}_t(\delta)}\|x_t\|_{V_{t-1}^{-1}(\lambda_t)}\\
    &\leq 2\sqrt{\bar{\beta}_{t-1}(\delta)}\|x_t\|_{V_{t-1}^{-1}(\lambda_t)}\\
    & \sr{(b)}{\le} 2\sqrt{\bar{\beta}_{t-1}(\delta)}\|x_t\|_{V_{t-1}^{-1}(\lambda_T)}.
\end{align*}
$(a)$ comes from using the confidence ellipsoid such that $\langle x_t,\hat{\theta}_t\rangle\leq \langle x_t,\theta^*\rangle+\sqrt{\bar{\beta}_t(\delta)}\|x_t\|_{V_{t-1}^{-1}(\lambda_t)}$. $(b)$ follows since $\lambda_t$ is a decreasing sequence in $t,$ this implies that $\lambda_t\geq \lambda_T>0$ for all $t$ and in turn, $\|x_t\|_{V_{t-1}^{-1}(\lambda_t)}\leq \|x_t\|_{V_{t-1}^{-1}(\lambda_T)}$ and
 $\bar{\beta}_T(\delta)\geq 1+\beta_t(\delta)\geq 1$.
For $t_0\geq 1$, let $H_T=\{t\in \{t_0,\dots,T\}: \|x_t\|_{V_{t-1}^{-1}(\lambda_T)}^2> 1\}$ and let $\barH_T := \{t_0,\ldots,T\} \sm H_T$ and $(T-t_0)=|\barH_T|+|H_T|$, it follows
\begin{align*}
    \sum_{t=t_0}^T r_t&\leq \sum_{t=t_0}^T r_t\{\|x_t\|_{V_{t-1}^{-1}(\lambda_T)}^2\leq1\}+  \sum_{t=t_0}^T r_t\{\|x_t\|_{V_{t-1}^{-1}(\lambda_T)}^2>1\}\\
    & \leq \sqrt{|\barH_T|\sum_{t\in \barH_{T}} r_t^2}+\sum_{t\in H_T} 2S_* \\
    & \leq 2\sqrt{|\barH_{T}|\sum_{t\in J\barH_{T}}\bar{\beta}_{T}(\delta)\|x_t\|^2_{V_{t-1}^{-1}(\lambda_T)}}+2S_*|H_T| \\
    & \sr{(c)}{\le} 2\sqrt{(T-(t_0-1))\bar{\beta}_T(\delta)}\sqrt{2d\log\left(1+\frac{T-(t_0-1)}{d\lambda_T}\right)}+4S_*\frac{d}{\log(2)}\log\left(1+\frac{1}{\lambda_T\log(2)}\right)
\end{align*} 
where $\sqrt{\bar{\beta}_T(\delta)}= R\sqrt{2d\log\left(\frac{T^2\pi^2 (1+T/\lambda_T)}{6\delta}\right)}+R\sqrt{d}$ and $(c)$ follows from Lemma \ref{lem:epc} and the elliptical potential lemma.
Therefore, summing the bounds, $\Reg_T=\sum_{t=1}^Tr_t=\sum_{t=1}^{t_0-1}r_t+\sum_{t=t_0}^{T}r_t$ we obtain our desired results.
\end{proof}

\begin{proof}[\textbf{Proof of Corollary~\ref{cor:changing}}]
For each case let $t_0=\min\{t\geq 1: \lambda_t^{1/2}S\leq R\sqrt{d}\}$.
\begin{itemize}
    \item[1.] Consider $\lambda_t=\frac{R^2 d}{\log^\gam(1+t)},$ for $\gam>0$, $ t \geq 0$
then,  \[t_0=\exp(S_*^{2/\gamma})-1\]
Now we insert into the regret bound from Theorem \ref{Thm: Confidence_ell} and it follows that w.p at least $1-\dt,$
\begin{align*}
&\Reg_T\leq 2\exp(S_*^{2/\gamma})S_*\\
&+\sqrt{T+2}\left(R\sqrt{2d\log\left(\frac{T^2\pi^2(1+\fr{T\log^\gam (1+T)}{dR^2})}{6\delta}\right)}+R\sqrt{d}\right)\sqrt{2d\log\left(1+\frac{\log^\gam (1+T)(T-\exp(S_*^{2/\gam})+2)}{R^2 d^2}\right)}\\&+4S_*\frac{d}{\log(2)}\log\left(1+\frac{\log^\gam(1+T)}{d^2 R^2\log(2)}\right)\\
&\leq \tcO(Rd\sqrt{T} + S_*(d + \exp(S_*^{2/\gam}))).
\end{align*}

\item[2.] Consider $\lambda_t=\frac{R^2 d}{t^\alpha}$ where for this case \[t_0= S_*^{2/\alpha}.\] Now we insert into the regret bound from Theorem \ref{Thm: Confidence_ell} and it follows that w.p at least $1-\dt,$ \begin{align*}
\Reg_T\leq &2S_*^{2/\alpha+1}+\sqrt{T-S_*^{2/\alpha}+1}\left(R\sqrt{2d\log\left(\frac{T^2\pi^2(1+T^{\alpha+1}/(dR^2))}{6\delta}\right)}+R\sqrt{d}\right)\sqrt{2d\log\left(1+\frac{T^\alpha(T-S^{2/\alpha}+1)}{R^2 d^2}\right)}\\&+4S_*\frac{d}{\log(2)}\log\left(1+\frac{T^\alpha}{R^2d^2\log(2)}\right)\\
&\leq \tcO(Rd\sqrt{(\alpha^2+1)T} + S_*(d  \alpha + S_*^{2/\alpha}))
\end{align*}
\item[3.] Consider $\lambda_t=\frac{R^2 d}{\exp(t^q)}$ for $q\in(0,1)$, then 

\[t_0=\log^{1/q}(S_*^2).\]

Now we insert into the regret bound from Theorem \ref{Thm: Confidence_ell} and it follows that w.p at least $1-\dt,$ 
\begin{align*}
&\Reg_T
\\&\leq 2\log^{1/q}(S_*^2)S_*
  \\&\phantom{\le} + \sqrt{T}\left(R\sqrt{2d\log\left(\frac{T^2\pi^2(1+T \exp{(T^q)}/(dR^2))}{6\delta}\right)}+R\sqrt{d}\right)\sqrt{2d\log\left(1+\frac{\exp(T^q)(T-\log^{1/q}(S_*^2)+1)}{R^2 d^2}\right)}
  \\&\phantom{\le} + 4S_*\frac{d}{\log(2)}\log\left(1+\frac{\exp(T^q)}{R^2d^2\log(2)}\right)
\\&\leq \tcO(Rd\sqrt{T^{1+q}}+S_*(d T^q + \log^{1/q}(S_*^2))).
\end{align*}

\end{itemize}

\end{proof}

\subsection{Elliptical Potential Count Lemma}
We present our our elliptical potential count lemma in Lemma~\ref{lem:epc} below, which can also be found in a parallel work of~\citet{kim21improved}.
We also show a tight and easy derivation for bounding implicit inequalities like $X \le A \log(1 + BX)$

The original form of Lemma~\ref{lem:epc} appears in~\citet[Exercise 19.3]{lattimore20bandit}, but Lemma~\ref{lem:epc} has an improved leading constant and has an arguably more streamlined proof.
We remark that a similar result appeared earlier in \citet[Proposition 3]{russo13eluder} as well.

\begin{lemma}\label{lem:epc} 
(Elliptical potential count)
Let $x_1,\dots,x_T\in\mathbb{R}^d$ be a sequence of vectors with $\|x_t\|_2\leq 1$ for all $t\in [T]$. Let $V_t(\lambda)=\lambda I+\sum_{s=1}^t x_sx_s^\T$, and define the following, $H_T=\{t\in [T]: \|x\|_{V_{t-1}^{-1}(\lambda)}^2> 1\}$ and denote the size of $H_T$ by $|H_T|$, then for $t\in [T]$ \[|H_T|\leq \frac{2d}{\log(2)}\log\left(1+\frac{1}{\lambda\log(2)}\right).\]
\end{lemma}
\begin{proof}
Let $M_T=V_T^{H_T}(\lambda)=\lambda I+\sum_{s=1}^{T}\one\{\|x_s\|_{V_{s-1}^{-1}}\geq 1\}x_sx_s^\T =\lambda I+\sum_{s\in H_T} x_sx_s^\T $. 
The determinant of $M_T$ is 
\[|M_T|=|\lambda I+\sum_{s\in H_T} x_sx_s^\T|=\Pi_{i=1}^d(\lambda+\lambda_i),\] where $\lambda_i$ are the eigenvalues of $\sum_{s\in H_T} x_sx_s^\T$. We now bound above, 
\begin{align*}
    |M_T|&=|\lambda I+\sum_{s\in H_T} x_sx_s^\T|\\
    &=\Pi_{i=1}^d(\lambda+\lambda_i)\\
    &\leq \left(\sum_{i=1}^d \left(\frac{\lambda+\lambda_i}{d}\right)\right)^d \tag{AM - GM}\\
    & = \left(\frac{1}{d}\text{trace}(M_T)\right)^d\\
    & \leq \left(\frac{1}{d}(\text{trace}(V)+\text{trace}(\sum_{s=1}^T \one\{s\in H_T\}x_sx_s^\T)\right)^d\\
    & \leq \left(\frac{1}{d}(d\lambda+\sum_{s=1}^T \one\{s\in H_T\}x_sx_s^\T)\right)^d\\
    &\leq \left(\frac{d\lambda+|H_T|}{d}\right)^d
\end{align*}
To get a lower bound, we start with the following form of $M_T$, as shown in \citet[proof of Lemma 11]{ay11improved} ,
\begin{align*}
    |M_T|&=|\lam I|\Pi_{k\in H_T}(1+\|x_k\|_{M_{k-1}^{-1}}^2)\\
    & \geq |\lam I|\Pi_{k\in H_T}(1+\|x_k\|_{V_{k-1}^{-1}(\lam)}^2)\\
    & = \lam^d \Pi_{k\in H_T}(1+\|x_k\|_{V_{k-1}^{-1}(\lam)}^2)\\
    & \geq \lambda^d 2^{|H_T|}. \tag{Since $\|x_k\|_{V_{k-1}^{-1}(\lam)}\geq 1$ for all $k\in H_T$ and $\Pi_{i=1}^n 2=2^n$} 
\end{align*}
The first inequality is true because $\|x_k\|^2_{V_{k-1}^{-1}(\lam)}$ is a nonincreasing sequence in $k$; i.e., $\|x_k\|_{M_{k-1}^{-1}}\geq\|x_k\|_{V_{k-1}^{-1}(\lam)}.$
Thus, we have 
\[\lambda^d 2^{|H_T|}\leq \left(\frac{d\lambda+|H_T|}{d}\right)^d.\]
Taking the log of both sides and rearranging we obtain
\[|H_T|\leq \frac{d}{\log(2)}\log\left(1+\frac{|H_T|}{\lambda d}\right).\]
Using Lemma~\ref{lem:solvelog} with $\eta=1/2$, $A=\frac{d}{\log(2)}$, $B=\frac{1}{d\lambda}$, and $X=|H_T|$, 
\[|H_T|\leq \frac{2d}{\log(2)}\log\left(\frac{2d}{2\log(2)}\left(\frac{1}{|H_T|}+\frac{1}{\lambda d}\right)\right)=\frac{2d}{\log(2)}\log\left(\frac{1}{\log(2)}\left(\frac{d}{|H_T|}+\frac{1}{\lambda }\right)\right).\]
We now consider two cases; when for some value $c>0$, $\frac{|H_T|}{d}\leq c$ and $\frac{|H_T|}{d}\geq c.$\\
\textbf{Case 1:} Suppose that $|H_T|<dc$. Then it follows that
\[|H_T|\leq 2d\log\left(1+\frac{|H_T|}{d\lambda}\right)\leq 2d\log\left(1+\frac{c}{\lambda}\right)\]

\textbf{Case 2:} Suppose that $|H_T|\geq dc$ then 
 \begin{align*}
    |H_T|&\leq \frac{2d}{\log(2)}\log\left(\frac{1}{\log(2)}\left(\frac{d}{|H_T|}+\frac{1}{\lambda }\right)\right)\\
    & \leq \frac{2d}{\log(2)}\log\left(\frac{1}{\log(2)}\left(\frac{1}{c}+\frac{1}{\lambda }\right)\right)
\end{align*}
Choose $c=\frac{1}{\log(2)}$ and since the right hand side is independent of $t$, then for up to time $T$ for all cases we have that
\[|H_T|\leq \frac{2d}{\log(2)}\log\left(1+\frac{1}{\log(2)\lambda }\right).\]
\end{proof}

We now present Lemma~\ref{lem:solvelog} that is a key inequality used in the proof above.
We remark that one can make the leading constant 1 by applying the bound to the RHS of $X \le A \log(1+BX)$:
\begin{align*}
    X\leq A \log\del{1+B\cd \del{\inf_{\eta\in(0,1)}\frac{1}{1-\eta}A\log\left(\frac{A}{2\eta}\left(\frac{1}{X}+B\right)\right)}}
\end{align*}

\begin{lemma}\label{lem:solvelog}
Let $X,A,B\geq 0$.
Then, $X\leq A \log(1+BX)$ implies that 
\begin{align*}
    X\leq \inf_{\eta\in(0,1)}\frac{1}{1-\eta}A\log\left(\frac{A}{2\eta}\left(\frac{1}{X}+B\right)\right)
\end{align*}
\end{lemma}
\begin{proof}
Let $\eta \in (0,1)$
\begin{align*}
    &X\leq A \log(1+BX)\\
    &\Rightarrow X\leq A\log(\frac{\eta X}{A})+A\log\left(\frac{A}{\eta}\left(\frac{1}{X}+B\right)\right)\\
    &\Rightarrow X \leq A(\frac{\eta X}{A}-1)+A\log\left(\frac{A}{\eta}\left(\frac{1}{X}+B\right)\right) \tag{since $\log(x)\leq x-1$}\\
    &\Rightarrow X(1-\eta)\leq A\log\left(\frac{A}{2\eta}\left(\frac{1}{X}+B\right)\right) \\
    & X\leq \frac{1}{1-\eta}A\log\left(\frac{A}{2\eta}\left(\frac{1}{X}+B\right)\right) \tag{since $\eta<1$}
\end{align*}
\end{proof}

Note that the technique shown above is more generic and can be used to solve implicit inequalities like $X \le O(\log(X))$.\footnote{
  
}
For example, it is easy to show that
\begin{align*}
    X \le A \log(X) + C \implies X \le \inf_{\eta\in(0,1)} \frac{1}{1-\eta} \del{A \log\del{\frac{A}{\eta e}} + C}
\end{align*}
In fact, one can apply the same technique to solve $X^p \le O(\log(X))$ where $p>0$.
The proof simply introduced a free parameter $\eta$ that is tight with the best choice of $\eta$.
This proof is much simpler than other known techniques such as~\citet[Appendix B]{antos10active}.

\section{Proof of Theorem \ref{thm:olsoful}}
\def\OLS{{\mathsf{OLS}}}

\begin{lemma}\label{lem: warmup}
In OLSOFUL (Algorithm~\ref{alg:olsoful}), the length of the warmup phase $\tau$ satisfies
\[\tau =\cO(d\ln(d))\]
\end{lemma}
\begin{proof} 
Define $\blue{\barV(\pi)}:=\sum_{x\in\cX} \pi_x x x^\T$ where $\pi\in\Delta^{|\cX|}$.
Recall that $\tau$ is the length of the warmup phase.
Let $\blue{\om} := |\cX_0|$ be the number of time steps spent on the first stage of the warmup procedure. 
Let $\blue{\pi_\om} \in \Delta^K$ have $1/(2d)$ for $x \in \cX_0$ and $0$ otherwise. Let
   \[\barV_\tau= \barV_\om+\sum_{s=\om+1}^\tau x_s x_s^\T\] 
where $\barV_\om := 2d\cdot \barV(\pi_\om)= \sum_{x\in \cX_0} xx^\T$. 
Thus, $\barV_\om$ is the initial matrix for $\barV_\tau$ defined by the $2d$ samples used in Algorithm \ref{alg:olsoful}. 

Let $\pi^* = \arg\min_{\pi \in \Delta^{|\cX|}} \max_{x\in\cX}  \|x\|^2_{\barV(\pi)}$, the solution of the G-optimal design problem.
By~\citet{kiefer60theequivalence}, $\pi^* = \arg\max_{\pi \in \Delta^{|\cX|}} \max_{x\in\cX} |\barV(\pi)|$ as well. This implies that $|\barV_\tau| = \tau^d|\fr{1}{\tau}\barV_\tau| \le \tau^d|\barV(\pi^*)|$.
Then, $|\barV_{\om}|  = \om^d |\fr1{\om} \barV_{\om}|$.
According to proof of Theorem 4.2 of \cite{kumar05minimum}, we have
\begin{equation} \label{KYthm:4.2}
  \ln |\barV(\pi^*)| - \ln |\barV(\pi_\omega)| = \mathcal{O}(d\ln(d))
\end{equation}
which means that the first stage of the warmup allocates arms so the resulting (normalized) determinant is sufficiently close to the ideal allocation of $\pi^*$.

We will now show $\tau\leq 2d+\mathcal{O}(d\log(d))$, thus, using the bounds on the determinant of $\log|\barV_\tau|$ and $\log|\barV_\om|$  
\begin{align*}
  \ln |\barV_\tau|  - \ln | \barV_\om | &\le d\ln(\tau ) + \ln(|\barV(\pi^*)|)  -  \ln|\fr1\om \barV(\om)| - d \ln (\om)   
  \\&\le d\ln(\tau ) + O(d\ln(d)) - d\ln (\om) 
\end{align*}
where the last inequality is by (\ref{KYthm:4.2}).
On the other hand,
\begin{align*}
  \ln|\barV_{\tau}| - \ln|\barV_\omega|
  = \sum_{t=\omega+1}^{\tau} \ln(1 + \|x_t\|^2_{\barV_{t-1}^{-1}} )
  ~> (\tau-\omega)\ln(2)
\end{align*}
Therefore, we have
\begin{align*}
  (\tau-\omega)\ln(2) < \mathcal{O}(d\ln(d)) +  d\ln\left(\fr{\tau}{\om}\right).
\end{align*}
Dropping $\ln(2)$ on the LHS we can find $\tau.$ Thus,
\begin{align*}
    (\tau-\om)\log(2)&< \mathcal{O}(d\ln(d)) +  d\ln\left(\fr{\tau}{\om}\right)\\
    &\Rightarrow \tau<  \mathcal{O}(d\ln(d)) +  \fr{d}{\ln(2)}d\ln\left(\fr{2 d\tau}{2 d\om}\right) +\om\\
    &\Rightarrow \tau<  \mathcal{O}(d\ln(d)) + \fr{d}{\ln(2)}\ln\left(\fr{ \tau}{2 d}\right) +\fr{d}{\ln(2)}\ln\left(\fr{2d}{\om}\right) +2d \tag{$\om\leq 2d$}\\
    &\Rightarrow \tau<  \mathcal{O}(d\ln(d)) + \fr{d}{\ln(2)}\left(\fr{ \tau}{2d}-1\right) +\fr{d}{\ln(2)}\ln\left(\fr{2d}{\om}\right) +2d \tag{$\ln(x)\le x-1$ }\\
    &\Rightarrow \tau(\fr{2\ln(2)-1}{2\ln(2)})<  \mathcal{O}(d\ln(d)) +\fr{d}{\ln(2)}\ln\left(\fr{d}{\eta\om}\right) +2d\\
    & \Rightarrow \tau < \mathcal{O}(d\ln(d)) + \fr{4\ln(2)}{2\ln(2)-1}d+ \fr{4\ln(2)}{2\ln(2)-1}d \ln\left(\fr{d}{\eta\om}\right) = \cO(d\ln(d))~. 
\end{align*}
This concludes the proof.
\end{proof}

\begin{theorem}
\label{thm: FAS conf ellip}
Suppose $\cX\subset \RR^d$. Define $\barV_\tau:=\sum_{s=1}^{\tau} x_{s} x_{s}^\T$ where $\tau$ is the length of the warmup procedure. Then, for any $\delta>0$, we have, with at least probability $1-2\delta$ and for all $t\geq 0$ that
\[\th^* \in C_t =\left\{\th^*\in\mathbb{R}^d:\ 
    \|\hth_{t,0} - \th^*\|_{\barV_{t}}\leq \sqrt{\beta_{t}^\OLS}:= R \del{\sqrt{\ln\del{\fr{|\barV_t|}{|\barV_\tau|\dt^2}}} + \sqrt{2\ln(2)d + 4\ln(1/\dt)}}
\right\}\]
where $\hth_{t,0}$ is the estimator defined for $\tau+1,\dots,t$, such that 
\begin{equation}\label{eq:mle_fixxed_arm}
    \hth_{t,0} = \arg\min_{\th} \sum_{s=1}^t \fr12(y_s - x_s^\T\th)^2.
    \end{equation}
\end{theorem}
\begin{proof}
We start by defining notation. Define $\tau$ to be the length of the warmup procedure. Let $\blue{\hth_{\tau,0}}$ be the MLE up to $\tau$. Let $X_{\le\tau} \in \RR^{\tau \times d}$ ($X_{>\tau} \in \RR^{(t-\tau) \times d}$) be the design matrix consisting of the pulled arms during the warmup phase (after the warmup up time step $t$, respectively). 
Recall $\blue{\barV_\tau} := \sum_{s=1}^{\tau} x_{s} x_{s}^\T$ is the covariance matrix. 
Let $\eta_{\le \tau} = (\eta_{1},\ldots, \eta_\tau)^{\T}$ and $\eta_{>\tau} = (\eta_{\tau+1},\ldots, \eta_t)^{\T}$.
For the following rounds after $\tau$, define the MLE estimator \begin{equation}
\label{estimator}
  \hth_{t,0} = \arg \min_\th \fr12 \sum_{s=\tau+1}^{t} (x_s^\T \th - y_s)^2 + \fr{1}{2}\|\th - \hth_{\tau,0}\|^2_{\barV_{\tau}}.
\end{equation}
Fix $x\in\RR^d$.
The optimality condition can be written as
\begin{align*}
  \barV_t(\hth_{t,0} - \th^*) = \del{X_{>\tau}^\T X_{>\tau} + \barV_\tau}(\hth_{t,0} - \th^*) = X_{>\tau}^\T \eta_{>\tau} -  \barV_\tau(\th^* - \hth_{\tau,0}) 
\end{align*}

In fact, this optimality condition for Estimator (\ref{estimator}) is exactly that of using the whole data. That is, Estimator (\ref{estimator}) has the same optimality condition as Estimator (\ref{eq:mle_fixxed_arm}).

Fix $x\in\RR^d$, which we will specify later on.
Then, 
  \begin{equation} \label{eq: optimality_con}
      x^\T(\hth_{t,0} - \th^*) 
  = x^\T \barV_t^{-1} X_{>\tau}^\T \eta_{>\tau} - x^\T \barV_t^{-1}\barV_\tau(\th^* - \hth_{\tau,0}).
  \end{equation}

Consider the first term on the RHS of Equation (\ref{eq: optimality_con}). We will bound it using \citet[Theorem 1]{ay11improved}. We will use the high probability bound,
\begin{equation}\label{eq: HPB_1}
\PP\left( \forall t\geq 0: \|X_{>\tau}^\T \eta_{>\tau} \|_{\barV_t^{-1}} \leq R\sqrt{ \ln\del{\fr{|\barV_t|}{|\barV_\tau|}\fr1{\dt^2} } } \right)\geq 1-\delta.
\end{equation}
Therefore, with probability at least $1-\delta$ for all $t\geq \tau+1,$
 \begin{align*}  
 x^\T \barV_t^{-1} X_{>\tau}^\T \eta_{\le \tau}&\le \|x\|_{\barV_t^{-1}} \| X^{\T}_{>\tau}\eta_{>\tau}\|_{\barV_t^{-1}}\\
 &\sr{(a)}{\le}\|x\|_{\barV_t^{-1}} R\sqrt{\ln(\fr{|\barV_t|}{|\barV_\tau|\dt^2} )}.
 \end{align*}
 where $(a)$ is due to the high probability bound in Equation (\ref{eq: HPB_1}).

One can further bound the adaptive design by 
\begin{align*}
  \ln\del{\fr{|\barV_t|}{|\barV_\tau|\dt^2}}
  \le \ln\del{\fr{|\barV_\tau^{-1/2} \barV_t \barV_\tau^{-1/2}|}{\dt^2}}
  &\le d\ln\del{\fr1d \tr(\barV_\tau^{-1/2} \barV_t \barV_\tau^{-1/2}) \fr{1}{\dt^2}}
  \\&= d\ln\del{ \fr1d \sum_{s=1}^t \tr(\barV_\tau^{-1/2} x_s x_s^\T \barV_\tau^{-1/2}) \fr{1}{\dt^2}}
  \\&\sr{(b)}{\le} d\ln\del{ \fr {t} d \fr{1}{\dt^2}}
\end{align*}
where $(b)$ is due to the fact that $\max_{x\in\cX}\tr(\barV_\tau^{-1/2} x x^\T \barV_\tau^{-1/2})=\max_{x\in\cX}  \|x\|^2_{\barV_\tau^{-1}} \le 1$.

 Consider the second term on the RHS of Equation (\ref{eq: optimality_con}). We use \citet[Theorem 1]{ay11improved}. We have the high probability event 
 \begin{equation}\label{eq: HPB_2}
 \PP\left( \forall t\geq 0: \|X_{\le\tau}^\T \eta_{\le\tau} \|_{[(1 + \sigma)\barV_\tau]^{-1}} \leq R\sqrt{ \ln\del{\fr{|(1+\sigma)\barV_\tau|}{|\sigma \barV_\tau|}\fr1{\dt^2} } } \right)\geq 1-\delta.
 \end{equation}
 Therefore with probability at least $1-\dt$ for all $t\geq \tau+1$
\begin{align*}
  - x^\T \barV_t^{-1} \bar{V}_\tau(\th^* - \hth_{\tau,0})&= - x^\T \barV_t^{-1} \barV_\tau(\th^* - \barV_\tau^{-1} X_{\le \tau}^\T X_{\le \tau} \th^* - \barV_\tau^{-1} X_{\le \tau}^\T \eta_{\le \tau})
  \\&= - x^\T \barV_t^{-1} \barV_\tau(-\barV_\tau^{-1} X_{\le\tau}^\T \eta_{\le\tw})
  \\&=  x^\T \barV_t^{-1} X_{\le \tau}^\T \eta_{\le\tau}
  \\&\le \|x\|_{\barV_t^{-1}} \|X_{\le \tau}^\T \eta_{\le \tau} \|_{\barV_t^{-1}} 
  \\&\le  \|x\|_{\barV_t^{-1}} \|X_{\le\tau}^\T \eta_{\le\tau} \|_{\barV_\tau^{-1}} 
  \\&= \|x\|_{\barV_t^{-1}} \sqrt{1+\sigma}\|X_{\le\tau}^\T \eta_{\le\tau} \|_{[(1 + \sigma)\barV_\tau]^{-1}}  \tag{for any $\sigma > 0$}
  \\&\sr{(c)}{\le} \|x\|_{\barV_t^{-1}} \sqrt{1+\sigma} R\sqrt{ \ln\del{\fr{|(1+\sigma)\barV_\tau|}{|\sigma \barV_\tau|}\fr1{\dt^2} } }
  \\&\le \|x\|_{\barV_t^{-1}} \sqrt{1+\sigma} R\sqrt{ d\ln\del{\fr{1+\sigma}{\sigma} } + 2\ln(1/\dt) }.
\end{align*}
where $(c),$ is due to the high probability bound in Equation (\ref{eq: HPB_2}).

We now choose $\sigma = 1$, which gives us 
\[\sqrt{1+\sigma} R\sqrt{ d\ln\del{\fr{1+\sigma}{\sigma} } + 2\ln(1/\dt) }\le R \sqrt{2\ln(2) d + 4\ln(1/\dt)}.\]
Finally, we combine the two high probability bounds, Equation's (\ref{eq: HPB_1}) and (\ref{eq: HPB_2}) to bound the RHS of Equation (\ref{eq: optimality_con}). Then we have with probability at least $1-2\dt$,
\begin{align*}
   | x^\T(\hth_{t,0} - \th^*) |  \leq \|x\|_{\barV_t^{-1}}R \del{\sqrt{\ln\del{\fr{|\barV_t|}{|\barV_\tau|\dt^2}}} + \sqrt{2\ln(2)d + 4\ln(1/\dt)}}
\end{align*}
Choosing $x$ we obtain our final result. That is, set $x=\barV_t(\hth_{t,0}-\th^*)$. This implies  $| x^\T(\hth_{t,0} - \th^*) |= \|\hth_{t,0}-\th^*\|_{\barV_t}^2$ and  $\|\barV_t(\hth_{t,0}-\theta^*)\|_{\barV_{t}^{-1}}=\|\hth_{t,0}-\th^*\|_{\barV_t}$.Then with probability at least $1-2\delta$,
\begin{align*}
    \|\hth_{t,0} - \th^*\|_{\barV_{t}}\leq R \del{\sqrt{\ln\del{\fr{|\barV_t|}{|\barV_\tau|\dt^2}}} + \sqrt{2\ln(2)d + 4\ln(1/\dt)}}
\end{align*}

\end{proof}

\begin{proof}[\textbf{Proof of Theorem~\ref{thm:olsoful}}]

Recall $\om = |\cX_0|$ and $\tau$ is the length of the total warmup procedure in Algorithm \ref{alg:olsoful}. Define $r_t := \max_{x\in\cX} \la x, \th^*\ra - \la x_t, \th^*\ra $ and recall $\Reg_T=\sum_{t=1}^T r_t$. Then,
\begin{align*}
  \Reg_T&= \sum_{t=1}^\tau r_t + \sum_{t=\tau+1}^T r_t
  \\&\sr{(a)}{\le} 2dS_* + S_* \cO(d\ln(d)) + \sum_{t=\tau+1}^T  r_t
\end{align*}
where $(a)$ is due to Lemma \ref{lem: warmup}.
It remains to bound the last term. On the high probability event in Theorem \ref{thm: FAS conf ellip},
\[C_t=\left\{\th^*\in\mathbb{R}^d:\ 
    \|\hth_{t,0} - \th^*\|_{\barV_{t}}\leq \sqrt{\beta_{t}^\OLS}\right\}\] we bound the following,
\begin{align*}
  \sum_{t=\tau+1}^T  r_t
  &= \sum_{t=\tau+1}^T \del{ \max_{x\in\cX} \la x, \th^*\ra - \la x_t, \th^*\ra} 
  \\&\le \sum_{t=\tau+1}^T \del{ \la  x_t,\til\theta_{t,0}  \ra - \la x_t, \th^*\ra} 
  \\&= \sum_{t=\tau+1}^T \del{ \la  x_t, \til\theta_{t,0} -\hth_{t,0}  \ra - \la x_t, \th^* - \hth_{t,0}\ra} 
  \\&\le \sum_{t=\tau+1}^T  \|x_t\|_{\barV_{t-1}^{-1}} \del{ \|\til\th_{t-1,0} - \hth_{t-1,0}\|_{\barV_{t-1}} +  \|\th^* - \hth_{t-1,0}\|_{\barV_{t-1}}}
  \\&\sr{(b)}{\le} \sum_{t=\tau+1}^T  \|x_t\|_{\barV_{t-1}^{-1}} 2 \sqrt{\beta^{\OLS}_{t-1}}
  \\&\le 2\sqrt{T \sum_{t=\tau+1}^T \|x_t\|^2_{\barV_{t-1}^{-1}} \beta_{t-1}^{\OLS} }
  \\&\le 2\sqrt{T \sum_{t=\tau+1}^T  2\ln(1+\|x_t\|^2_{\barV_{t-1}^{-1}}) \beta^{\OLS}_{t-1} }
  \\&= 2\sqrt{2} \sqrt{T \ln\del{\fr{|\barV_T|}{|\barV_\tau|} }  \beta^{\OLS}_{t-1} }
  \\&\leq \mathcal{O}\del{R \sqrt{dT\ln(\frac{T}{d}) \ln(\frac{T}{d\dt})}} 
\end{align*}
where $(b)$ uses the high probability event of $C_t$. Hence the desired regret bound holds with at least probability $1-2\dt.$ 
\end{proof}

\subsection{Computationally Efficient Version of Algorithm \ref{alg:warmup_psuedo}}

As discussed in Section 5, we present a computationally efficient version of Algorithm \ref{alg:warmup_psuedo}, which is displayed in Algorithm~\ref{alg:warmup}. 
That is, for choosing the direction, $b_i$ 
by utilizing the Gram-Schmidt Process to find a vector that lives in the orthogonal complement of $\Psi$. 
It is easy to verify that the computational complexity of Algorithm~\ref{alg:warmup} is  $\cO(d^2|\cX|)$.
\begin{algorithm}
  \caption{Computationally efficient BH algorithm~\cite{betke93approximating}}
  \label{alg:warmup}
  \begin{algorithmic}
    \STATE \textbf{Input:} the set of points $\cX \subset \RR^d$ with $|\cX| = K$.
    \STATE If $K \le 2d$, then $\cX_0 \larrow \cX$. Return.
    \STATE $\Psi \larrow\{0\}$, $\cX_0 \larrow \emptyset$, $i \larrow 0$, $v_0 \larrow (0,\ldots,0)^\T\in\RR^d$.
    \WHILE{$\RR^d \setminus \Psi\neq \emptyset$}
    \STATE $i \larrow i + 1$
    \IF {$i=1$ }
    \STATE Set $b_i = e_1$
    \ELSE
    \STATE Set $v^\perp_{i-1} = v_{i-1} - \sum_{j\le i-2} \fr{\la v^\perp_j, v_{i-1} \ra}{\la v_j^\perp,  v_j^\perp \ra}v_j^\perp$
    \STATE Set $b_i = e_i - \sum_{j\le i-1} \fr{\la v^\perp_j, e_i\ra}{\la v_j^\perp, v_j^\perp \ra}v_j^\perp$
    \ENDIF
    \STATE $p \larrow \arg \max_{x\in\cX} \la b_i, x \ra$; $\cX_0 \larrow \cX_0 \cup \{p\}$ 
    \STATE $q \larrow \arg \min_{x\in\cX} \la b_i, x \ra$; $\cX_0 \larrow \cX_0 \cup \{q\}$ 
    \STATE $v_i \larrow p - q$
    \STATE $\Psi \larrow \mathsf{Span}(\Psi, \{v_i\})$
    \ENDWHILE
  \end{algorithmic}
\end{algorithm}

%

\end{document}